\documentclass{article}

\usepackage[preprint]{neurips_2026}

\usepackage[utf8]{inputenc} % allow utf-8 input
\usepackage[T1]{fontenc}    % use 8-bit T1 fonts
\usepackage{hyperref}       % hyperlinks
\usepackage{url}            % simple URL typesetting
\usepackage{booktabs}       % professional-quality tables
\usepackage{amsfonts}       % blackboard math symbols
\usepackage{nicefrac}       % compact symbols for 1/2, etc.
\usepackage{microtype}      % microtypography
\usepackage[table]{xcolor}  % colors

\usepackage{graphicx}
\usepackage{amsmath}
\usepackage{amsthm}
\usepackage{algorithm}
\usepackage{algorithmic}
\usepackage{multirow}
\usepackage{diagbox}
\usepackage{colortbl}
\usepackage{makecell}

\newtheorem{theorem}{Theorem}[section]
\newtheorem{proposition}[theorem]{Proposition}

\title{CAMPA: Efficient and Aligned Multimodal Graph Learning via Decoupled Propagation and Aggregation}

\author{%
\textbf{Daohan Su \quad Hao Liu \quad Xunkai Li \quad Yinlin Zhu} \\
\textbf{Xiong Yongfu \quad Yi Liu \quad Hongchao Qin \quad Rong-Hua Li \quad Guoren Wang}
}

\begin{document}

\maketitle

\begin{abstract}
Multimodal Graph Neural Networks (MGNNs) have shown strong potential for learning from multimodal attributed graphs, yet most existing approaches rely on tightly coupled architectures that suffer from prohibitive computational overhead.
In this paper, we present a systematic empirical analysis showing that decoupled MGNNs are substantially more efficient and scalable for large-scale graph learning.
However, we identify a critical bottleneck in existing decoupled pipelines, namely \textbf{modal conflict}, which arises in both the propagation and aggregation stages.
Specifically, independent multi-hop diffusion causes cross-modal semantic divergence during propagation, while naive fusion fails to align multi-hop feature trajectories during aggregation, jointly limiting effective representation learning.
To address this challenge, we propose \textbf{CAMPA}, a \textbf{C}ross-modal \textbf{A}ligned \textbf{M}ultimodal \textbf{P}ropagation \& \textbf{A}ggregation framework for decoupled multimodal graph learning.
Concretely, CAMPA introduces a two-stage alignment mechanism:
(1) cross-modal aligned propagation, which injects cross-modal similarity priors into message passing to preserve semantic consistency without additional parameter overhead;
(2) trajectory aligned aggregation, which leverages trajectory-level self-attention and cross-attention to capture and align long-range dependencies across modalities and hops.
Extensive experiments on diverse benchmark datasets and tasks demonstrate that CAMPA consistently outperforms strong coupled and decoupled baselines while preserving the efficiency advantages of the decoupled paradigm.
\end{abstract}

\section{Introduction}
\label{sec:intro}
Learning on Multimodal Attributed Graphs (MAGs) has attracted growing attention as multimodal data and relational structures are increasingly intertwined in real-world systems, such as recommendation platforms and social media networks~\cite{LGMRec,bian2020_directed_app_social1,zhu2025mm_graph,yan2025magb,wan2026openmag}.
Meanwhile, graph neural networks (GNNs) have become a dominant paradigm for relational representation learning due to their strong ability to model high-order structural dependencies~\cite{kipf2016gcn,velivckovic2017gat,chen2020gcnii,wu2020gnn_survey1,zhou2022gnn_survey2}.
Unlike conventional GNNs that mainly operate on single-modality node attributes, multimodal GNNs (MGNNs) must jointly model heterogeneous semantic signals (e.g., text and images) together with graph topology, making effective multimodal fusion under structural constraints the central challenge in this domain~\cite{cai2024omg_nas,liu2025graph_mllm,wang2025mg_llm}.

A predominant paradigm in existing MGNNs relies on \textit{coupled} architectures, where feature propagation and modality transformation are tightly intertwined within each layer~\cite{MMGCN,MGAT}.
Although this design enables repeated topology-aware multimodal interaction, it also makes training expensive as graph size grows.
This raises the first question: can the decoupling principle, which separates graph propagation from feature transformation and has been widely adopted for scalable graph learning~\cite{wu2019sgc,frasca2020sign,zhang2021ndls}, provide a more efficient foundation for MAG learning?
To answer this question, we compare representative coupled baselines (i.e., MGAT~\cite{MGAT} and MMGCN~\cite{MMGCN}) with decoupled baselines (e.g., MSGC, SIGN~\cite{frasca2020sign}, and MGDN~\cite{hu2024mgdcf}) under the same MAG setting.
As shown in Fig.~\ref{fig:empirical}(a), decoupled models reach comparable or better accuracy with substantially less training time, whereas coupled MGNNs incur much higher latency and quickly saturate.
This observation motivates decoupled architectures as a natural starting point for scalable multimodal graph learning, especially as recent MAG research moves toward larger graphs, more diverse objectives, stronger multimodal backbones, and graph-conditioned multimodal generation~\cite{fan2025mlaga,fang2025graphgpt_o,ning2025graph4mm,jin2024instructg2i,wan2026openmag}.
The complete experimental protocol is provided in Appendix~\ref{app:empirical_setup}.

However, decoupling also raises a second question: what limitation emerges when multimodal interaction is weakened during the propagate-then-aggregate pipeline?
In decoupled MAG models, modality-specific features are first diffused over the graph and then summarized by a lightweight fusion module.
This pipeline can introduce modal conflict at two stages.
\textit{(1) Propagation stage:} text and image features originate from different semantic spaces, and independent multi-hop diffusion can make their hop-wise representations drift along modality-specific directions, producing spatial asynchrony.
\textit{(2) Aggregation stage:} after propagation, each modality forms a multi-hop trajectory, but naive late fusion struggles to match semantically corresponding hops across modalities, leading to trajectory-level mismatch.
To examine whether these two stages are indeed bottlenecks, we equip representative decoupled baselines with lightweight propagation-stage and aggregation-stage alignment modules.
Fig.~\ref{fig:empirical}(b) shows that both alignments improve decoupled baselines, and their combination yields the strongest gains.
These intervention results motivate our central insight: scalable MAG learning requires not only decoupled computation, but also explicit stage-wise alignment during both propagation and aggregation.
Additional heatmap and t-SNE visualizations in Appendix~\ref{app:empirical_vis} further illustrate the representation-level divergence behind this phenomenon.

\begin{figure}[t]
\centering
\includegraphics[width=\linewidth]{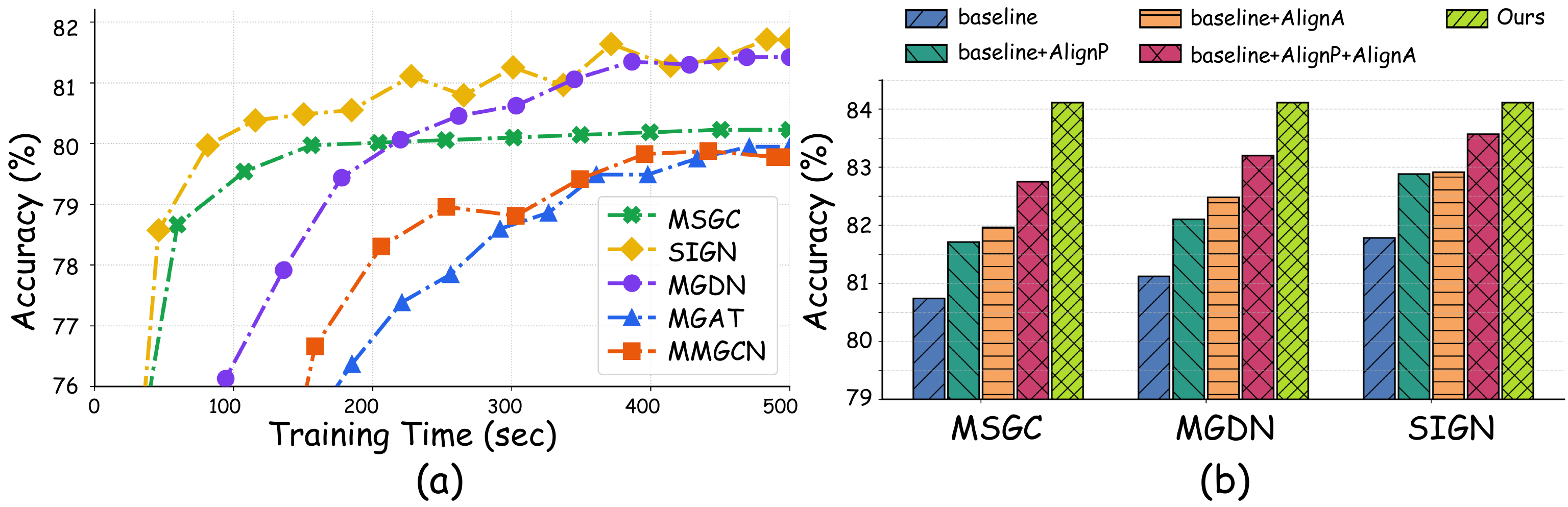}
\vspace{-0.5cm}
\caption{Empirical motivation. (a) Accuracy vs. training time, verifying the scalability advantage of decoupled architectures in MAG learning. (b) Performance boost of decoupled baselines after lightweight propagation-stage alignment (AlignP), aggregation-stage alignment (AlignA), or both, motivating explicit stage-wise coordination.}
% \vspace{-0.5cm}
\label{fig:empirical}
\end{figure}

Motivated by these critical insights, we propose \textbf{CAMPA} (\textbf{C}ross-modal \textbf{A}ligned \textbf{M}ultimodal \textbf{P}ropagation \& \textbf{A}ggregation), a decoupled framework designed to resolve modal conflict while preserving the efficiency advantages of decoupled MGNNs.
The key idea is to align multimodal information at the two stages where conflict arises.
First, CAMPA introduces \textit{cross-modal aligned propagation}, which injects cross-modal similarity priors into the diffusion process to maintain semantic consistency during propagation without additional parameter overhead.
Second, CAMPA further employs \textit{trajectory aligned aggregation}, which leverages trajectory-level self-attention and cross-attention to align multi-hop contextual dependencies across modalities during aggregation.
In this way, CAMPA strengthens cross-modal coordination while retaining the lightweight computation pattern that makes decoupled MGNNs attractive for large-scale learning.

Our contributions are summarized as follows:
\underline{\textit{(1) In-depth empirical insights.}}
We provide empirical evidence showing the scalability advantage of decoupled MGNNs and identify stage-wise modal conflict as a key bottleneck in both propagation and aggregation stages.
\underline{\textit{(2) A novel decoupled alignment framework.}}
We propose CAMPA, a principled MAG learning framework that aligns multimodal information at both propagation and aggregation stages, thereby reconciling semantic alignment with the efficiency advantages of the decoupled paradigm.
\underline{\textit{(3) Strong performance across tasks.}}
Extensive experiments on diverse benchmark datasets and tasks demonstrate that CAMPA consistently outperforms baselines while maintaining high efficiency.

\section{Preliminaries}
\label{sec:prelim}

\subsection{Problem Formulation}
We consider a MAG $\mathcal{G}=(\mathcal{V},\mathcal{E},\{\mathbf{X}^{(m)}\}_{m\in\mathcal{M}})$, where $\mathcal{V}$ and $\mathcal{E}$ are the node and edge sets, $\mathcal{M}$ is the modality index set, and $\mathbf{X}^{(m)}\in\mathbb{R}^{N\times d_m}$ is the feature matrix of modality $m$.
The shared topology is represented by the adjacency matrix $\mathbf{A}\in\{0,1\}^{N\times N}$ with degree matrix $\mathbf{D}$ and normalized form $\widetilde{\mathbf{A}}=\mathbf{D}^{-1/2}\mathbf{A}\mathbf{D}^{-1/2}$.
For clarity, we instantiate the formulation on the common text-image setting, although it naturally extends to more modalities.

Following recent MAG studies and recent benchmark efforts~\cite{zhu2025mm_graph,yan2025magb,wan2026openmag}, we consider two families of downstream tasks.
\textbf{Graph-centric tasks} examine whether the learned node representations preserve both graph topology and multimodal semantics, including \textit{Node Classification}, \textit{Node Clustering}, and \textit{Link Prediction}.
\textbf{Modality-centric tasks} examine whether graph-aware representations improve cross-modal understanding and generation, including \textit{Modality Retrieval}, \textit{Modality Matching}, \textit{Modality Alignment}, and \textit{Graph-to-Image Generation}~\cite{yoon2023mmgl,ning2025graph4mm,fang2025graphgpt_o,jin2024instructg2i}.

\subsection{Decoupled Graph Learning}
Decoupled graph learning has become an important paradigm for scalable graph learning by explicitly separating structure propagation from feature transformation~\cite{wu2019sgc,frasca2020sign,chen2022nagphormer,zhang2021ndls}.
Unlike coupled GNNs, which interleave neighborhood aggregation and parameterized transformation layer by layer, decoupled methods first compute multi-hop propagated features and then apply lightweight predictors on top.
This separation substantially reduces training cost while preserving strong predictive performance on large graphs~\cite{gamlp}.
In a standard decoupled pipeline, the input feature matrix $\mathbf{X}$ is transformed into a trajectory $\{\mathbf{H}^{(0)},\mathbf{H}^{(1)},\ldots,\mathbf{H}^{(K)}\}$, where $\mathbf{H}^{(0)}=\mathbf{X}$ and $\mathbf{H}^{(k)}$ denotes the $k$-hop propagated representation.
These trajectories are then aggregated by a learnable predictor, making decoupled graph learning attractive when propagation can be precomputed or reused across epochs.

\subsection{Multimodal Attributed Graph Learning}
MAG learning extends graph learning to MAGs by jointly modeling graph topology and modality attributes~\cite{zhu2025mm_graph,yan2025magb,wan2026openmag}.
Existing MAG methods can be broadly summarized into three lines.

\textbf{Conventional multimodal graph models} typically adapt early multimodal recommendation architectures to graphs by designing topology-aware interaction or fusion modules~\cite{MMGCN,MGAT,LGMRec}.
These methods mainly emphasize topology-guided multimodal fusion and form the early foundation.

\textbf{LLM- or token-based methods} incorporate graph topology into multimodal reasoning through neighborhood tokenization, topology-aware aligners, or graph-conditioned generation modules~\cite{ning2025graph4mm,fang2025graphgpt_o,fan2025mlaga,jin2024instructg2i,liu2025graph_mllm,wang2025mg_llm}.
This line connects MAG learning with multimodal comprehension and generation by transforming structural information into tokenized contexts or auxiliary semantic priors.

\textbf{Graph-enhanced methods} further design specialized architectures, such as refined message passing, graph transformers, spectral filtering, or unified embedding spaces, to better capture topology-modality interplay~\cite{MIG-GT,hu2025ntsformer,DGF,he2025unigraph2,cai2024omg_nas,li2025c_mag}.
Compared with the previous two lines, these methods place greater emphasis on graph backbones and more expressive topology-aware multimodal interaction.

\section{Method}
\label{sec:method}

\begin{figure}[t]
\centering
\includegraphics[width=\linewidth]{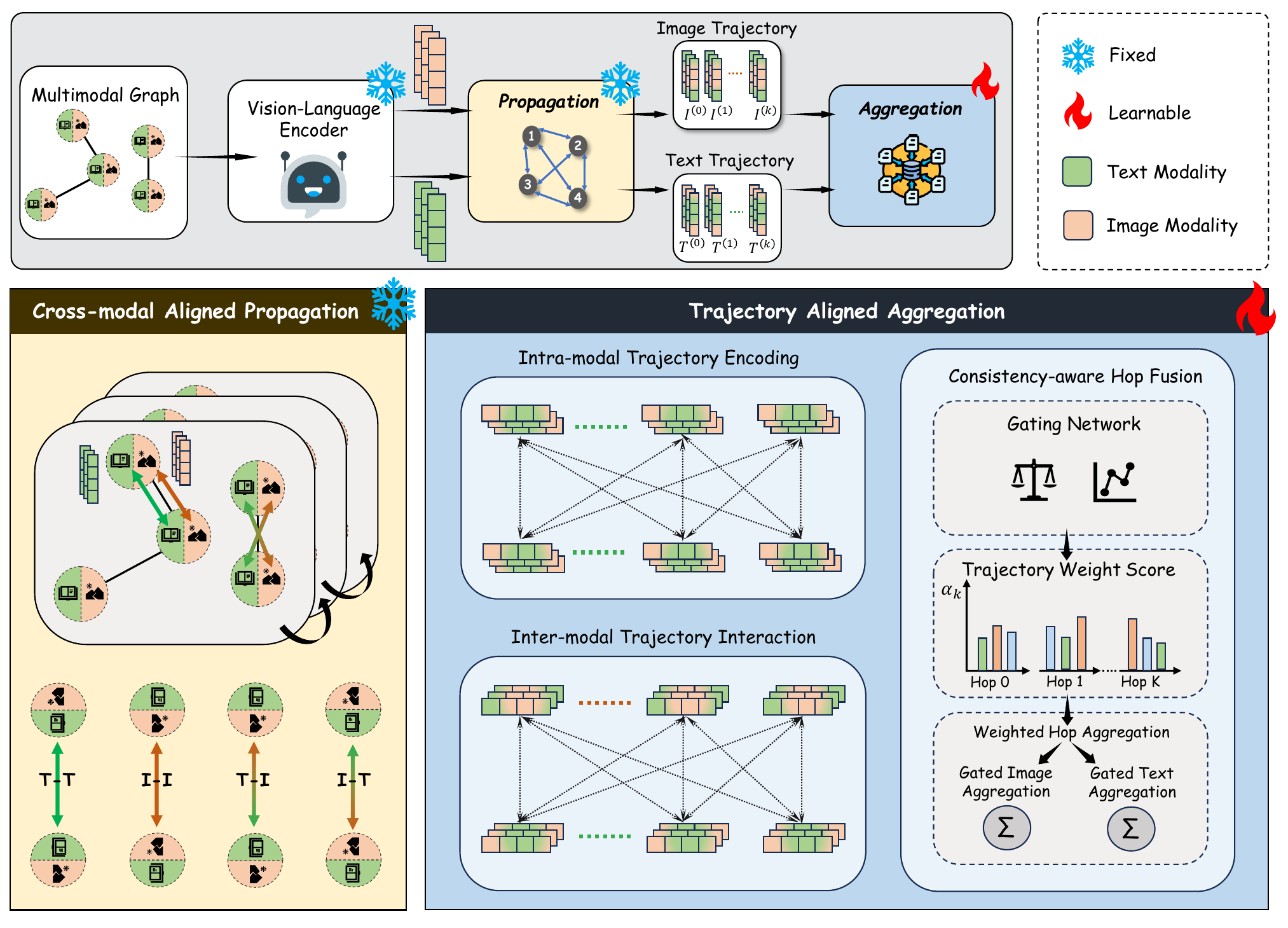}
\vspace{-0.5cm}
\caption{Overall framework of CAMPA. It first performs \textit{cross-modal aligned propagation} to construct modality-specific multi-hop trajectories with cross-modal semantic priors injected into graph diffusion. It then applies \textit{trajectory aligned aggregation} to model intra-modal hop dependencies and inter-modal trajectory interactions, producing unified node representations for downstream tasks.}
\label{fig:framework}
\end{figure}

We instantiate CAMPA on the common text-image setting, while the formulation naturally extends to more modalities.
The raw inputs consist of graph topology together with node-associated textual and visual contents.
Before graph learning, the text and image contents are encoded by a frozen pre-trained vision-language encoder, producing modality features that serve as the input of CAMPA.
If the encoded modalities have different dimensionalities, lightweight modality-specific projections are further applied; for simplicity, we denote the resulting graph features as $\mathbf{H}_{t}^{(0)}, \mathbf{H}_{i}^{(0)} \in \mathbb{R}^{N \times d}$.
As illustrated in Fig.~\ref{fig:framework}, CAMPA follows a decoupled yet alignment-aware pipeline with two coordinated alignment components:
\textbf{Cross-modal Aligned Propagation (CAP)} injects cross-modal semantic priors into graph diffusion, and
\textbf{Trajectory Aligned Aggregation (TAA)} performs trajectory-level interaction and consistency-aware fusion over the propagated features.

\subsection{Cross-modal Aligned Propagation}
Existing decoupled MAG models typically diffuse each modality independently on the shared graph.
While efficient, such a design ignores whether connected nodes are semantically compatible across modalities, causing the propagated trajectories to drift apart as the hop depth increases.
CAP addresses this issue by refining the propagation operators with modality-aware semantic priors, while preserving the precomputable nature of decoupled diffusion.

\textbf{Semantic Propagation Priors.}
For each modality pair $(u,v)\in\{t,i\}\times\{t,i\}$, we estimate edge-wise semantic priors from the initial node features and use them to modulate the shared topology:
\begin{equation}
\mathbf{S}_{uv} =\left(1+\mathrm{cos\mbox{-}sim}\!\left(\mathbf{H}_{u}^{(0)}, \mathbf{H}_{v}^{(0)}\right)\right)/2, \qquad
\widetilde{\mathbf{A}}_{uv} = \mathrm{Norm}\!\left(\mathbf{A}\odot\mathbf{S}_{uv}\right),
\label{eq:cap_prior}
\end{equation}
where the rescaling maps cosine similarity to $[0,1]$, $\mathrm{Norm}(\mathbf{B})=\mathbf{D}_{\mathbf{B}}^{-1/2}\mathbf{B}\mathbf{D}_{\mathbf{B}}^{-1/2}$ is symmetric normalization, and $\odot$ denotes the Hadamard product.
Here, $\widetilde{\mathbf{A}}_{tt}$ and $\widetilde{\mathbf{A}}_{ii}$ preserve modality-specific local smoothness, whereas $\widetilde{\mathbf{A}}_{ti}$ and $\widetilde{\mathbf{A}}_{it}$ inject cross-modal correspondence into graph diffusion.

\textbf{Aligned Multi-hop Diffusion.}
Let $\bar{u}$ denote the counterpart modality of $u$.
For each $u\in\{t,i\}$, CAP updates the $k$-hop representation by combining intra-modal diffusion, cross-modal correction, and residual retention:
\begin{equation}
\mathbf{H}_{u}^{(k+1)} =
\alpha_{u}\widetilde{\mathbf{A}}_{uu}\mathbf{H}_{u}^{(k)}
+ \beta_{u}\widetilde{\mathbf{A}}_{u\bar{u}}\mathbf{H}_{\bar{u}}^{(k)}
+ \gamma_{u}\mathbf{H}_{u}^{(0)},
\quad
\alpha_{u}+\beta_{u}+\gamma_{u}=1,
\label{eq:cap_update}
\end{equation}
where $\alpha_{u},\beta_{u},\gamma_{u}\ge 0$ are tunable hyperparameters.
The first term captures modality-specific structural homophily, the second term continuously corrects the trajectory with aligned evidence from the other modality, and the residual branch stabilizes deep propagation.
After $K$ steps, we obtain two aligned propagation trajectories, $\{\mathbf{H}_{t}^{(0)},\ldots,\mathbf{H}_{t}^{(K)}\}$ and $\{\mathbf{H}_{i}^{(0)},\ldots,\mathbf{H}_{i}^{(K)}\}$.
Notably, CAP introduces no hop-specific transformation during propagation, and thus remains fully compatible with the efficiency advantages of decoupled MAG learning.

\subsection{Trajectory Aligned Aggregation}
CAP produces alignment-aware multi-hop trajectories, yet directly averaging or concatenating them still overlooks two key aspects:
different hops within one modality often play complementary roles, and cross-modal semantic agreement may emerge at different propagation depths.
To address this, TAA models both intra-modal trajectory dependency and inter-modal trajectory interaction before performing adaptive hop fusion.

\textbf{Intra-modal Trajectory Encoding.}
For each modality $u\in\{t,i\}$, we slightly abuse notation and stack the hop-wise propagated features as $\mathbf{H}_{u}=[\mathbf{H}_{u}^{(0)},\ldots,\mathbf{H}_{u}^{(K)}]\in\mathbb{R}^{N\times (K+1)\times d}$.
All subsequent attention operations are applied along the trajectory dimension, while the node dimension is omitted as an implicit batch dimension.
TAA first applies self-attention to capture dependencies among shallow and deep propagated features within the same modality:
\begin{equation}
\mathbf{Q}_{u} = \mathbf{H}_{u}\mathbf{W}_{Q}^{u}, \quad
\mathbf{K}_{u} = \mathbf{H}_{u}\mathbf{W}_{K}^{u}, \quad
\mathbf{V}_{u} = \mathbf{H}_{u}\mathbf{W}_{V}^{u},
\label{eq:taa_qkv_self}
\end{equation}
\begin{equation}
\mathbf{F}_{u} =
\mathbf{H}_{u}
+ \mathrm{MLP}\!\left(
\mathrm{softmax}\!\left(
\frac{\mathbf{Q}_{u}\mathbf{K}_{u}^{\top}}{\sqrt{d}}
\right)\mathbf{V}_{u}
\right),
\label{eq:taa_self}
\end{equation}
where $\mathbf{W}_{Q}^{u},\mathbf{W}_{K}^{u},\mathbf{W}_{V}^{u}\in\mathbb{R}^{d\times d}$ are learnable projections.
This step contextualizes each hop with information from the entire propagation trajectory rather than treating different diffusion depths independently.

\textbf{Inter-modal Trajectory Interaction.}
Based on the self-contextualized trajectories, we further introduce cross-attention to exchange complementary semantics across modalities:
\begin{equation}
\widehat{\mathbf{Q}}_{u\leftarrow\bar{u}} = \mathbf{F}_{u}\mathbf{W}_{Q}^{u\leftarrow\bar{u}}, \quad
\widehat{\mathbf{K}}_{u\leftarrow\bar{u}} = \mathbf{F}_{\bar{u}}\mathbf{W}_{K}^{u\leftarrow\bar{u}}, \quad
\widehat{\mathbf{V}}_{u\leftarrow\bar{u}} = \mathbf{F}_{\bar{u}}\mathbf{W}_{V}^{u\leftarrow\bar{u}},
\label{eq:taa_qkv_cross}
\end{equation}
\begin{equation}
\mathbf{G}_{u} =
\mathbf{F}_{u}
+ \mathrm{MLP}\!\left(
\mathrm{softmax}\!\left(
\frac{\widehat{\mathbf{Q}}_{u\leftarrow\bar{u}}
\widehat{\mathbf{K}}_{u\leftarrow\bar{u}}^{\top}}{\sqrt{d}}
\right)\widehat{\mathbf{V}}_{u\leftarrow\bar{u}}
\right),
\label{eq:taa_cross}
\end{equation}
where $\mathbf{W}_{Q}^{u\leftarrow\bar{u}},\mathbf{W}_{K}^{u\leftarrow\bar{u}},\mathbf{W}_{V}^{u\leftarrow\bar{u}}\in\mathbb{R}^{d\times d}$.
Unlike post-hoc late fusion, Eq.~\eqref{eq:taa_cross} aligns the two modalities at the trajectory level, enabling each hop to retrieve semantically compatible evidence from the other modality.

\textbf{Consistency-aware Hop Fusion.}
After trajectory interaction, we estimate hop-wise cross-modal consistency and use it as an auxiliary signal for modality-specific hop weighting:
\begin{equation}
\mathbf{C} = \mathrm{Cos}\!\left(\mathbf{G}_{t}, \mathbf{G}_{i}\right)\in\mathbb{R}^{N\times (K+1)\times 1},
\label{eq:taa_consistency}
\end{equation}
where $\mathrm{Cos}(\cdot,\cdot)$ computes cosine similarity along the feature dimension for aligned hop slices.
The hop weights of the two modalities are then predicted by separate gating networks:
\begin{equation}
\mathbf{A}_{u} = \mathrm{Softmax}_{\mathrm{hop}}\!\left(\mathcal{G}_{u}\!\left([\mathbf{G}_{u}\|\mathbf{C}]\right)\right), \quad u\in\{t,i\},
\label{eq:taa_gate}
\end{equation}
where $\mathcal{G}_{t}(\cdot)$ and $\mathcal{G}_{i}(\cdot)$ are modality-specific scoring networks, $\mathrm{Softmax}_{\mathrm{hop}}(\cdot)$ normalizes the scores along the hop dimension, and $\mathbf{A}_{u}^{(k)}\in\mathbb{R}^{N\times 1}$ denotes the hop weight for the $k$-th slice of modality $u$.
The final representations are obtained by weighted aggregation over the aligned trajectories:
\begin{equation}
\mathbf{Z}_{u}=
\sum_{k=0}^{K}
\mathbf{A}_{u}^{(k)} \odot \mathbf{G}_{u}^{(k)}, \quad u\in\{t,i\},
\label{eq:taa_output}
\end{equation}
and the final representation for downstream prediction is formed as $\mathbf{Z}=\mathbf{Z}_{t}\|\mathbf{Z}_{i}$.
This design preserves modality-specific aggregation preferences while still allowing the shared consistency signal to calibrate both gating processes.

\subsection{Theoretical Analysis}
\label{sec:theory}
We next formalize CAMPA from the perspectives of propagation stability and fusion control, with full proofs deferred to Appendix~\ref{app:theory}.

\subsubsection{Stable aligned diffusion in CAP}
Let
\begin{equation}
\mathbf{X}^{(k)}=
\begin{bmatrix}
\mathbf{H}_{t}^{(k)}\\
\mathbf{H}_{i}^{(k)}
\end{bmatrix}, \quad
\mathbf{P}=
\begin{bmatrix}
\alpha_{t}\widetilde{\mathbf{A}}_{tt} & \beta_{t}\widetilde{\mathbf{A}}_{ti}\\
\beta_{i}\widetilde{\mathbf{A}}_{it} & \alpha_{i}\widetilde{\mathbf{A}}_{ii}
\end{bmatrix}, \quad
\mathbf{R}=
\begin{bmatrix}
\gamma_{t}\mathbf{I} & \mathbf{0}\\
\mathbf{0} & \gamma_{i}\mathbf{I}
\end{bmatrix}.
\label{eq:theory_cap_block}
\end{equation}
Then Eq.~\eqref{eq:cap_update} can be written as
\begin{equation}
\mathbf{X}^{(k+1)}=\mathbf{P}\mathbf{X}^{(k)}+\mathbf{R}\mathbf{X}^{(0)}.
\label{eq:theory_cap_recursion}
\end{equation}
Since each $\widetilde{\mathbf{A}}_{uv}$ is symmetrically normalized, $\|\widetilde{\mathbf{A}}_{uv}\|_{2}\le 1$.
Define $\|\mathbf{X}\|_{2,\infty}=\max(\|\mathbf{X}_{t}\|_{F},\|\mathbf{X}_{i}\|_{F})$ for $\mathbf{X}=[\mathbf{X}_{t};\mathbf{X}_{i}]$ and let $\rho=\max(\alpha_{t}+\beta_{t},\alpha_{i}+\beta_{i})$.
Because $\alpha_{u}+\beta_{u}+\gamma_{u}=1$, the condition $\rho<1$ is naturally satisfied whenever $\gamma_{t}>0$ and $\gamma_{i}>0$.

\begin{theorem}[Convergence and stability of CAP]
\label{thm:cap_stability}
If $\rho<1$, the CAP recursion in Eq.~\eqref{eq:theory_cap_recursion} admits the unique fixed point
\begin{equation}
\mathbf{X}^{\star}=(\mathbf{I}-\mathbf{P})^{-1}\mathbf{R}\mathbf{X}^{(0)},
\label{eq:theory_cap_fixed}
\end{equation}
and its iterates satisfy
\begin{equation}
\|\mathbf{X}^{(k+1)}-\mathbf{X}^{\star}\|_{2,\infty}
\le
\rho \|\mathbf{X}^{(k)}-\mathbf{X}^{\star}\|_{2,\infty}.
\label{eq:theory_cap_contract}
\end{equation}
\end{theorem}

Theorem~\ref{thm:cap_stability} shows that CAP preserves the stability of decoupled propagation after introducing cross-modal correction, while the residual branch prevents unbounded drift.

\subsubsection{Discrepancy-controlled fusion in TAA}
We next analyze the final hop fusion induced by TAA.
For notational simplicity, we state the result for one node; since Eq.~\eqref{eq:taa_output} is applied row-wise, the same conclusion holds independently for every node.
Let $\mathbf{g}_{u}^{(k)}\in\mathbb{R}^{d}$ denote the $k$-th hop slice of modality $u$ after trajectory interaction, and let $a_{u}^{(k)}$ be the corresponding hop weight produced by Eq.~\eqref{eq:taa_gate}.
Then Eq.~\eqref{eq:taa_output} becomes
\begin{equation}
\mathbf{z}_{u}=\sum_{k=0}^{K} a_{u}^{(k)}\mathbf{g}_{u}^{(k)}, \quad u\in\{t,i\},
\label{eq:theory_taa_node}
\end{equation}
where $a_{u}^{(k)}\ge 0$ and $\sum_{k=0}^{K} a_{u}^{(k)}=1$ because the weights are normalized by softmax.

\begin{proposition}[Bounded and discrepancy-controlled fusion]
\label{prop:taa_bound}
For each modality $u\in\{t,i\}$, the fused representation $\mathbf{z}_{u}$ is a convex combination of $\{\mathbf{g}_{u}^{(0)},\ldots,\mathbf{g}_{u}^{(K)}\}$.
Moreover, the cross-modal discrepancy after hop fusion obeys
\begin{equation}
\|\mathbf{z}_{t}-\mathbf{z}_{i}\|_{2}
\le
\sum_{k=0}^{K} a_{t}^{(k)}\|\mathbf{g}_{t}^{(k)}-\mathbf{g}_{i}^{(k)}\|_{2}
+
\sum_{k=0}^{K}|a_{t}^{(k)}-a_{i}^{(k)}|\,\|\mathbf{g}_{i}^{(k)}\|_{2}.
\label{eq:theory_taa_gap}
\end{equation}
The same bound also holds after swapping the roles of $t$ and $i$.
\end{proposition}

Proposition~\ref{prop:taa_bound} clarifies the role of TAA.
The first term in Eq.~\eqref{eq:theory_taa_gap} captures the residual hop-wise mismatch after trajectory interaction, while the second term reflects the discrepancy between the two modality-specific gating distributions.
Accordingly, TAA improves fusion by jointly reducing trajectory mismatch and calibrating hop selection with the shared consistency signal.
Appendix~\ref{app:algo_complexity} further summarizes the end-to-end training procedure and computational complexity of CAMPA.

\section{Experiments}
\label{sec:exp}

We design the experiments to answer the following questions:
\textit{(Q1)} Can CAMPA consistently improve graph-centric prediction quality over both coupled and decoupled baselines?
\textit{(Q2)} Can the learned graph-aware multimodal representations also benefit modality-centric understanding tasks?
\textit{(Q3)} Which components in CAMPA contribute most to the final gains?
\textit{(Q4)} Does CAMPA preserve the efficiency advantages of decoupled MAG learning?
\textit{(Q5)} How sensitive is CAMPA to the propagation depth and alignment rate?

\subsection{Experimental Setup}

\textbf{Datasets.}
We evaluate CAMPA on a diverse collection of MAG datasets spanning recommendation, social, and content understanding scenarios.
Our experiments cover eight datasets, namely \texttt{Ele-fashion}, \texttt{Bili\_music}, \texttt{RedditS}, \texttt{Movies}, \texttt{Goodreads}, \texttt{Grocery}, \texttt{Cloth}, and \texttt{SemArt}, following recent MAG benchmark protocols~\cite{ni2019justifying,zhu2025mm_graph,yan2025magb}.
Detailed statistics are provided in Appendix~\ref{app:datasets}.

\textbf{Baselines.}
We compare CAMPA with four groups of representative baselines.
\textit{(1) General graph baseline} includes GAT~\cite{velivckovic2017gat}.
\textit{(2) Early multimodal graph baselines} include MMGCN~\cite{MMGCN} and MGAT~\cite{MGAT}.
\textit{(3) Recent strong multimodal baselines} include DMGC~\cite{DMGC}, DGF~\cite{DGF}, MIG-GT~\cite{MIG-GT}, LGMRec~\cite{LGMRec}, and UniGraph2~\cite{he2025unigraph2}.
\textit{(4) Decoupled baselines} include SIGN~\cite{frasca2020sign}, MSGC~\cite{wu2019sgc}, and MGDN~\cite{hu2024mgdcf}, where MSGC denotes our multimodal adaptation of the classic SGC-style decoupled pipeline.
Detailed descriptions of the compared baselines are provided in Appendix~\ref{app:baselines}.

\textbf{Evaluation Protocol.}
We follow standard task-specific protocols in recent MAG studies and report both graph-centric and modality-centric metrics accordingly.
All methods use the same data split and modality inputs on each dataset, while the detailed metric definitions, environment configuration, and hyperparameter settings are deferred to Appendices~\ref{app:eval_protocol} and~\ref{app:exp_settings}.

\subsection{Main Results}

To answer \textbf{Q1} and \textbf{Q2}, we report representative results in the main text and defer broader benchmark tables to Appendix~\ref{app:full_results}.
The two main tables jointly cover all eight datasets in our experiments.

\begin{table*}[t]
\centering
\caption{Performance comparison on representative graph-centric tasks. The best and second-best results are highlighted in \textbf{bold} and \underline{underline}, respectively.}
\label{tab:graph_results}
\footnotesize
\renewcommand{\arraystretch}{1.12}
\resizebox{\textwidth}{!}{
\setlength{\tabcolsep}{2.0mm}{
\begin{tabular}{l!{\vrule width 0.1pt}cc!{\vrule width 0.1pt}cc!{\vrule width 0.1pt}cc!{\vrule width 0.1pt}cc}
\Xhline{1pt}
\rowcolor{teal!72!black}
\multicolumn{1}{c!{\vrule width 0.1pt}}{\textbf{\textcolor{white}{Tasks}}} & \multicolumn{4}{c!{\vrule width 0.1pt}}{\textbf{\textcolor{white}{Node Classification}}} & \multicolumn{2}{c!{\vrule width 0.1pt}}{\textbf{\textcolor{white}{Link Prediction}}} & \multicolumn{2}{c}{\textbf{\textcolor{white}{Node Clustering}}} \\
\hline
\rowcolor{teal!12}
 & \multicolumn{2}{c!{\vrule width 0.1pt}}{\textbf{Goodreads}} & \multicolumn{2}{c!{\vrule width 0.1pt}}{\textbf{Grocery}} & \multicolumn{2}{c!{\vrule width 0.1pt}}{\textbf{Cloth}} & \multicolumn{2}{c}{\textbf{Movies}} \\
\cline{2-9}
\rowcolor{teal!12}
\multirow{-2}{*}{\diagbox[width=8.2em,height=2.3em]{\textbf{Methods}}{\textbf{Datasets}}}
 & Acc & F1 & Acc & F1 & MRR & Hits@3 & NMI & ARI \\
\hline
GAT & 58.94$_{\pm 0.23}$ & 53.02$_{\pm 0.18}$ & 81.01$_{\pm 0.21}$ & 72.94$_{\pm 0.24}$ & 42.89$_{\pm 0.25}$ & 51.05$_{\pm 0.28}$ & 40.21$_{\pm 0.28}$ & 31.48$_{\pm 0.33}$ \\
\hline
\rowcolor{teal!5} MMGCN & 60.11$_{\pm 0.22}$ & 54.56$_{\pm 0.24}$ & 81.44$_{\pm 0.24}$ & 73.46$_{\pm 0.19}$ & 44.11$_{\pm 0.23}$ & 52.13$_{\pm 0.26}$ & 44.15$_{\pm 0.39}$ & 34.02$_{\pm 0.30}$ \\
MGAT & 62.11$_{\pm 0.20}$ & 55.82$_{\pm 0.22}$ & 83.10$_{\pm 0.22}$ & \underline{75.21$_{\pm 0.25}$} & 45.41$_{\pm 0.30}$ & 54.37$_{\pm 0.24}$ & 46.15$_{\pm 0.36}$ & 37.81$_{\pm 0.40}$ \\
\hline
\rowcolor{teal!5} SIGN & 68.44$_{\pm 0.23}$ & 65.03$_{\pm 0.22}$ & 82.97$_{\pm 0.21}$ & 74.33$_{\pm 0.23}$ & 53.37$_{\pm 0.20}$ & 62.89$_{\pm 0.22}$ & 52.07$_{\pm 0.21}$ & 44.93$_{\pm 0.22}$ \\
MSGC & 64.35$_{\pm 0.19}$ & 61.47$_{\pm 0.21}$ & 81.96$_{\pm 0.22}$ & 73.98$_{\pm 0.22}$ & 51.01$_{\pm 0.23}$ & 59.82$_{\pm 0.27}$ & 49.72$_{\pm 0.22}$ & 41.50$_{\pm 0.22}$ \\
\rowcolor{teal!5} MGDN & 67.52$_{\pm 0.22}$ & 64.46$_{\pm 0.22}$ & 82.41$_{\pm 0.22}$ & 74.02$_{\pm 0.22}$ & 53.16$_{\pm 0.22}$ & 61.74$_{\pm 0.22}$ & 51.44$_{\pm 0.22}$ & 44.13$_{\pm 0.22}$ \\
\hline
DMGC & 65.18$_{\pm 0.21}$ & 62.86$_{\pm 0.23}$ & 81.35$_{\pm 0.23}$ & 71.31$_{\pm 0.18}$ & 53.21$_{\pm 0.22}$ & 61.98$_{\pm 0.25}$ & 51.39$_{\pm 0.37}$ & 44.31$_{\pm 0.28}$ \\
\rowcolor{teal!5} DGF & 70.21$_{\pm 0.19}$ & 64.05$_{\pm 0.22}$ & \underline{83.31$_{\pm 0.22}$} & 74.01$_{\pm 0.24}$ & 53.19$_{\pm 0.29}$ & 62.21$_{\pm 0.23}$ & \underline{52.18$_{\pm 0.34}$} & 45.01$_{\pm 0.39}$ \\
MIG-GT & 70.16$_{\pm 0.25}$ & 66.35$_{\pm 0.20}$ & 82.46$_{\pm 0.20}$ & 74.43$_{\pm 0.22}$ & 57.14$_{\pm 0.27}$ & 64.93$_{\pm 0.30}$ & 47.31$_{\pm 0.31}$ & 39.71$_{\pm 0.36}$ \\
\rowcolor{teal!5} LGMRec & 69.54$_{\pm 0.23}$ & 64.99$_{\pm 0.18}$ & 80.41$_{\pm 0.18}$ & 73.05$_{\pm 0.21}$ & \underline{59.25$_{\pm 0.25}$} & 48.31$_{\pm 0.28}$ & 52.01$_{\pm 0.28}$ & 44.48$_{\pm 0.33}$ \\
UniGraph2 & \underline{74.31$_{\pm 0.22}$} & \textbf{67.43$_{\pm 0.24}$} & 80.23$_{\pm 0.24}$ & 71.98$_{\pm 0.19}$ & 57.43$_{\pm 0.23}$ & \underline{65.17$_{\pm 0.26}$} & 51.83$_{\pm 0.39}$ & \underline{46.92$_{\pm 0.30}$} \\
\hline
\rowcolor{teal!5} CAMPA & \textbf{76.93$_{\pm 0.20}$} & \underline{67.29$_{\pm 0.22}$} & \textbf{84.98$_{\pm 0.21}$} & \textbf{81.01$_{\pm 0.23}$} & \textbf{60.33$_{\pm 0.30}$} & \textbf{68.17$_{\pm 0.24}$} & \textbf{56.53$_{\pm 0.36}$} & \textbf{48.93$_{\pm 0.40}$} \\
\Xhline{1pt}
\end{tabular}
}}
\end{table*}

\textbf{Graph-centric tasks.}
We first examine whether CAMPA improves graph-aware representations on graph-centric benchmarks.
These tasks directly test whether the learned embeddings preserve both structural signals and multimodal semantics.
Table~\ref{tab:graph_results} shows that CAMPA achieves the strongest overall performance across the selected settings.
For node classification, CAMPA attains the best accuracy on both \texttt{Goodreads} and \texttt{Grocery}, while also delivering the best F1 on \texttt{Grocery} and competitive F1 on \texttt{Goodreads}.
For link prediction on \texttt{Cloth}, CAMPA improves the best previous MRR from 59.25 to 60.33 and the best previous Hits@3 from 65.17 to 68.17.
For node clustering on \texttt{Movies}, CAMPA achieves the highest NMI and ARI, indicating that aligned propagation and trajectory-aware aggregation produce more coherent latent structures.

\begin{table*}[t]
\centering
\caption{Performance comparison on representative modality-centric tasks.}
\label{tab:modality_results}
\footnotesize
\renewcommand{\arraystretch}{1.12}
\resizebox{\textwidth}{!}{
\setlength{\tabcolsep}{2.0mm}{
\begin{tabular}{l!{\vrule width 0.1pt}cc!{\vrule width 0.1pt}cc!{\vrule width 0.1pt}cc!{\vrule width 0.1pt}cc}
\Xhline{1pt}
\rowcolor{teal!72!black}
\multicolumn{1}{c!{\vrule width 0.1pt}}{\textbf{\textcolor{white}{Tasks}}} & \multicolumn{4}{c!{\vrule width 0.1pt}}{\textbf{\textcolor{white}{Modality Retrieval}}} & \multicolumn{4}{c}{\textbf{\textcolor{white}{G2Image}}} \\
\hline
\rowcolor{teal!12}
 & \multicolumn{2}{c!{\vrule width 0.1pt}}{\textbf{Ele-fashion}} & \multicolumn{2}{c!{\vrule width 0.1pt}}{\textbf{RedditS}} & \multicolumn{2}{c!{\vrule width 0.1pt}}{\textbf{Bili\_music}} & \multicolumn{2}{c}{\textbf{SemArt}} \\
\cline{2-9}
\rowcolor{teal!12}
\multirow{-2}{*}{\diagbox[width=8.2em,height=2.3em]{\textbf{Methods}}{\textbf{Datasets}}}
 & T2I-MRR & I2T-MRR & T2I-MRR & I2T-MRR & CLIP & DINOv2 & CLIP & DINOv2 \\
\hline
GAT & 96.82$_{\pm 0.10}$ & 96.50$_{\pm 0.13}$ & 77.51$_{\pm 0.15}$ & 77.23$_{\pm 0.10}$ & 49.42$_{\pm 0.16}$ & 21.00$_{\pm 0.19}$ & 53.21$_{\pm 0.14}$ & 39.46$_{\pm 0.15}$ \\
\hline
\rowcolor{teal!5} MMGCN & 97.16$_{\pm 0.16}$ & 97.42$_{\pm 0.11}$ & 83.69$_{\pm 0.14}$ & 83.94$_{\pm 0.16}$ & 49.33$_{\pm 0.24}$ & 16.37$_{\pm 0.17}$ & 57.51$_{\pm 0.10}$ & 42.38$_{\pm 0.18}$ \\
MGAT & 98.35$_{\pm 0.15}$ & 98.81$_{\pm 0.17}$ & 85.58$_{\pm 0.12}$ & 85.76$_{\pm 0.15}$ & 49.56$_{\pm 0.21}$ & 19.36$_{\pm 0.25}$ & 58.31$_{\pm 0.14}$ & 43.51$_{\pm 0.19}$ \\
\hline
\rowcolor{teal!5} SIGN & 98.41$_{\pm 0.16}$ & 99.21$_{\pm 0.15}$ & 97.42$_{\pm 0.20}$ & 98.12$_{\pm 0.13}$ & 51.12$_{\pm 0.14}$ & 22.38$_{\pm 0.23}$ & 62.36$_{\pm 0.11}$ & 47.33$_{\pm 0.17}$ \\
MSGC & 97.32$_{\pm 0.13}$ & 97.74$_{\pm 0.17}$ & 92.42$_{\pm 0.15}$ & 92.33$_{\pm 0.20}$ & 49.35$_{\pm 0.21}$ & 20.14$_{\pm 0.10}$ & 59.87$_{\pm 0.15}$ & 46.15$_{\pm 0.17}$ \\
\rowcolor{teal!5} MGDN & 98.92$_{\pm 0.13}$ & 99.10$_{\pm 0.15}$ & 95.23$_{\pm 0.19}$ & 94.92$_{\pm 0.11}$ & 50.86$_{\pm 0.10}$ & 21.74$_{\pm 0.15}$ & 61.82$_{\pm 0.17}$ & 46.52$_{\pm 0.16}$ \\
\hline
DMGC & 95.40$_{\pm 0.15}$ & 94.83$_{\pm 0.10}$ & 90.36$_{\pm 0.13}$ & 90.24$_{\pm 0.15}$ & 49.62$_{\pm 0.23}$ & 22.23$_{\pm 0.16}$ & 62.15$_{\pm 0.15}$ & 48.71$_{\pm 0.15}$ \\
\rowcolor{teal!5} DGF & 95.50$_{\pm 0.14}$ & 96.33$_{\pm 0.16}$ & 97.88$_{\pm 0.11}$ & 97.74$_{\pm 0.14}$ & 47.35$_{\pm 0.20}$ & 18.36$_{\pm 0.24}$ & 64.50$_{\pm 0.11}$ & 48.83$_{\pm 0.15}$ \\
MIG-GT & 95.04$_{\pm 0.12}$ & 94.50$_{\pm 0.15}$ & 92.60$_{\pm 0.17}$ & 92.87$_{\pm 0.12}$ & 52.58$_{\pm 0.18}$ & \textbf{25.49$_{\pm 0.21}$} & \underline{65.45$_{\pm 0.15}$} & \underline{48.93$_{\pm 0.19}$} \\
\rowcolor{teal!5} LGMRec & \textbf{99.80$_{\pm 0.10}$} & \textbf{99.75$_{\pm 0.13}$} & \underline{99.84$_{\pm 0.15}$} & \textbf{99.87$_{\pm 0.10}$} & \underline{52.83$_{\pm 0.16}$} & 21.34$_{\pm 0.19}$ & 64.67$_{\pm 0.15}$ & 48.12$_{\pm 0.15}$ \\
UniGraph2 & 96.46$_{\pm 0.12}$ & 96.52$_{\pm 0.23}$ & 97.83$_{\pm 0.15}$ & 97.71$_{\pm 0.21}$ & 49.51$_{\pm 0.18}$ & 20.86$_{\pm 0.16}$ & 64.77$_{\pm 0.19}$ & 47.33$_{\pm 0.15}$ \\
\hline
\rowcolor{teal!5} CAMPA & \underline{99.69$_{\pm 0.15}$} & \underline{99.50$_{\pm 0.17}$} & \textbf{99.89$_{\pm 0.12}$} & \underline{99.79$_{\pm 0.15}$} & \textbf{53.16$_{\pm 0.21}$} & \underline{23.24$_{\pm 0.25}$} & \textbf{66.41$_{\pm 0.13}$} & \textbf{49.92$_{\pm 0.20}$} \\
\Xhline{1pt}
\end{tabular}
}}
\end{table*}

\textbf{Modality-centric tasks.}
Beyond graph-centric prediction, we further investigate whether the graph-aware multimodal representations learned by CAMPA can benefit downstream cross-modal understanding.
These tasks directly reflect whether the learned node representations are semantically coordinated across modalities rather than optimized only for structural prediction.
Table~\ref{tab:modality_results} shows that the benefits of CAMPA extend beyond graph-centric benchmarks.
For retrieval, CAMPA remains highly competitive on \texttt{Ele-fashion} and achieves the best T2I-MRR on \texttt{RedditS}, indicating that the learned graph-aware embeddings preserve strong bidirectional cross-modal correspondence.
For graph-to-image generation, CAMPA achieves the best CLIP score on both \texttt{Bili\_music} and \texttt{SemArt}, while also obtaining the best DINOv2 score on \texttt{SemArt}.
These results together show that CAMPA improves not only graph prediction but also modality-level understanding and generation quality.

\subsection{Ablation Study}
To answer \textbf{Q3}, we conduct a comprehensive ablation study.
The core question is whether stage-wise alignment is necessary at both stages, and whether the detailed design of trajectory interaction and modality-specific hop fusion is justified.
We consider the following variants, where SA and CA denote self-attention and cross-attention in TAA, respectively:
\textit{(1)} w/o CAP, where cross-modal aligned propagation is removed;
\textit{(2)} w/o TAA, where trajectory aligned aggregation is replaced by a simple trajectory fusion module;
\textit{(3)} w/o SA, where intra-modal trajectory encoding is removed;
\textit{(4)} w/o CA, where inter-modal trajectory interaction is removed;
\textit{(5)} full CAMPA.

\begin{table*}[t]
\centering
\caption{Ablation study on representative tasks.}
\label{tab:ablation}
\footnotesize
\renewcommand{\arraystretch}{1.12}
\resizebox{\textwidth}{!}{
\setlength{\tabcolsep}{2.0mm}{
\begin{tabular}{l!{\vrule width 0.1pt}cc!{\vrule width 0.1pt}cc!{\vrule width 0.1pt}cc!{\vrule width 0.1pt}cc}
\Xhline{1pt}
\rowcolor{teal!72!black}
\multicolumn{1}{c!{\vrule width 0.1pt}}{\textbf{\textcolor{white}{Tasks}}} & \multicolumn{2}{c!{\vrule width 0.1pt}}{\textbf{\textcolor{white}{Node Classification}}} & \multicolumn{2}{c!{\vrule width 0.1pt}}{\textbf{\textcolor{white}{Link Prediction}}} & \multicolumn{2}{c!{\vrule width 0.1pt}}{\textbf{\textcolor{white}{Modality Retrieval}}} & \multicolumn{2}{c}{\textbf{\textcolor{white}{G2Image}}} \\
\hline
\rowcolor{teal!12}
\textbf{Variants} & \multicolumn{2}{c!{\vrule width 0.1pt}}{\textbf{Goodreads}} & \multicolumn{2}{c!{\vrule width 0.1pt}}{\textbf{Cloth}} & \multicolumn{2}{c!{\vrule width 0.1pt}}{\textbf{RedditS}} & \multicolumn{2}{c}{\textbf{Bili\_music}} \\
\rowcolor{teal!12}
 & Acc & F1 & MRR & Hits@3 & T2I-MRR & I2T-MRR & CLIP & DINOv2 \\
\hline
w/o CAP & 71.18$_{\pm 0.24}$ & 65.31$_{\pm 0.23}$ & 56.52$_{\pm 0.19}$ & 63.33$_{\pm 0.25}$ & 96.33$_{\pm 0.16}$ & 96.72$_{\pm 0.22}$ & 50.42$_{\pm 0.20}$ & 20.95$_{\pm 0.21}$ \\
\rowcolor{teal!5} w/o TAA & 72.66$_{\pm 0.27}$ & 65.76$_{\pm 0.33}$ & 57.11$_{\pm 0.25}$ & 63.91$_{\pm 0.19}$ & 97.52$_{\pm 0.17}$ & 97.14$_{\pm 0.21}$ & 51.23$_{\pm 0.19}$ & 21.98$_{\pm 0.23}$ \\
w/o SA & 73.31$_{\pm 0.23}$ & 66.30$_{\pm 0.20}$ & 58.61$_{\pm 0.22}$ & 64.91$_{\pm 0.16}$ & 98.16$_{\pm 0.20}$ & 98.35$_{\pm 0.21}$ & 52.46$_{\pm 0.24}$ & 22.63$_{\pm 0.28}$ \\
\rowcolor{teal!5} w/o CA & 72.98$_{\pm 0.27}$ & 66.63$_{\pm 0.22}$ & 58.32$_{\pm 0.20}$ & 64.33$_{\pm 0.16}$ & 97.96$_{\pm 0.18}$ & 98.02$_{\pm 0.20}$ & 51.69$_{\pm 0.21}$ & 22.01$_{\pm 0.19}$ \\
\hline
\rowcolor{teal!5} CAMPA & \textbf{76.93$_{\pm 0.20}$} & \textbf{67.29$_{\pm 0.22}$} & \textbf{60.33$_{\pm 0.30}$} & \textbf{68.17$_{\pm 0.24}$} & \textbf{99.89$_{\pm 0.12}$} & \textbf{99.79$_{\pm 0.15}$} & \textbf{53.16$_{\pm 0.21}$} & \textbf{23.24$_{\pm 0.25}$} \\
\Xhline{1pt}
\end{tabular}
}
}
\end{table*}

The results in Table~\ref{tab:ablation} show that removing either CAP or TAA degrades all metrics, confirming that modal conflict should be handled in both propagation and aggregation stages.
Among the variants, w/o CAP causes the largest average drop, including 5.75 in accuracy on \texttt{Goodreads} and 4.84 in Hits@3 on \texttt{Cloth}.
This suggests that injecting cross-modal semantic priors during diffusion is critical for preventing modality discrepancy from being amplified along multi-hop propagation.
w/o TAA also clearly underperforms full CAMPA, indicating that aligned propagation trajectories still need trajectory-level interaction before fusion.
In addition, removing self-attention or cross-attention in TAA consistently weakens the results: w/o SA mainly affects graph-centric metrics, while w/o CA produces larger average degradation and is more harmful to modality-centric tasks.
These observations demonstrate that CAP and TAA are complementary, with CAP constructing aligned multi-hop trajectories and TAA coordinating intra-modal and inter-modal correspondence.

\subsection{Running Efficiency}
\label{sec:efficiency}

To answer \textbf{Q4}, efficiency is a key motivation for using a decoupled architecture.
Fig.~\ref{fig:efficiency} compares the performance-time trajectories of CAMPA and representative baselines on four dataset-task pairs, where the horizontal axis records wall-clock training time and the vertical axis reports the corresponding task metric.
Across all tasks, CAMPA rapidly enters the high-performance region and remains the best or highly competitive throughout training.
In contrast, SIGN is lightweight but saturates at a lower level, while stronger multimodal baselines such as UniGraph2 and MIG-GT generally require longer training and often remain below CAMPA under comparable time budgets.

\begin{figure}[t]
\centering
\includegraphics[width=\linewidth]{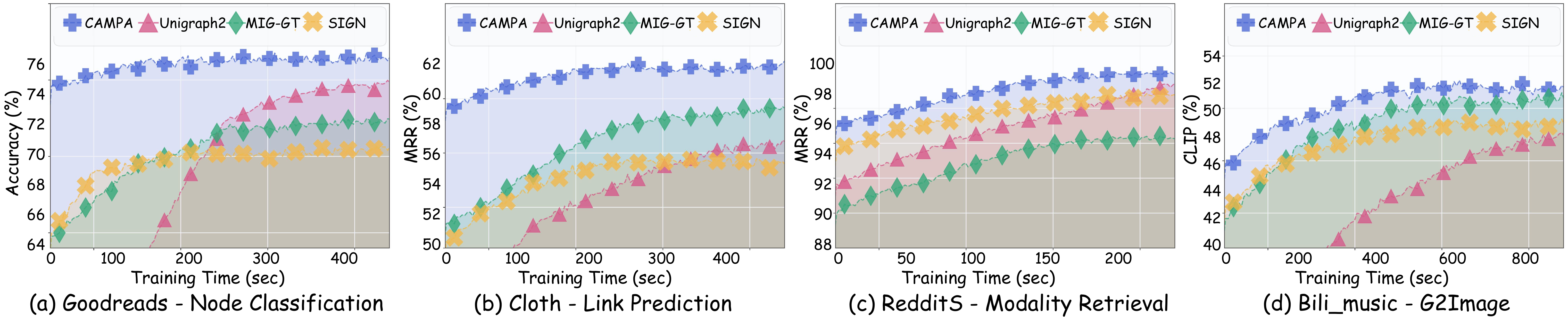}
\vspace{-0.5cm}
\caption{Efficiency analysis on four representative dataset-task pairs.
Each subplot reports task performance against wall-clock training time.
CAMPA consistently achieves a favorable performance-time trade-off across graph-centric and modality-centric tasks.}
\label{fig:efficiency}
\end{figure}

These results confirm that CAMPA preserves the lightweight optimization pattern of decoupled MAG models while substantially improving their representational quality.
The advantage is especially meaningful because CAMPA does not obtain higher accuracy by relying on a deeply coupled propagation process; instead, it injects alignment at the propagation and aggregation stages with limited additional cost.
Therefore, CAMPA provides a practical efficiency-performance trade-off for scalable MAG learning, complementing the complexity analysis in Appendix~\ref{app:algo_complexity}.

\subsection{Hyperparameter Sensitivity}

To answer \textbf{Q5}, we examine the sensitivity of CAMPA to two key hyperparameters: the propagation depth $K$ and the alignment rate $\beta$ that controls the strength of cross-modal propagation.
Fig.~\ref{fig:robustness} reports heatmaps on four representative dataset-task pairs, covering both graph-centric and modality-centric settings.
Across all datasets, setting $\beta=0$ usually leads to weaker performance, confirming that cross-modal propagation is important for effective decoupled MAG learning.
Increasing $\beta$ to a moderate positive range consistently improves the results, while the best region is generally broad rather than concentrated at a single configuration.
The effect of $K$ shows a similar pattern: shallow propagation may miss useful higher-order context, whereas overly deep propagation can introduce redundant or noisy signals.
Overall, CAMPA performs stably around $K=3$ or $4$ with $\beta\in[0.1,0.3]$, indicating that its gains do not rely on a narrow hyperparameter choice.

\begin{figure}[t]
\centering
\includegraphics[width=\linewidth]{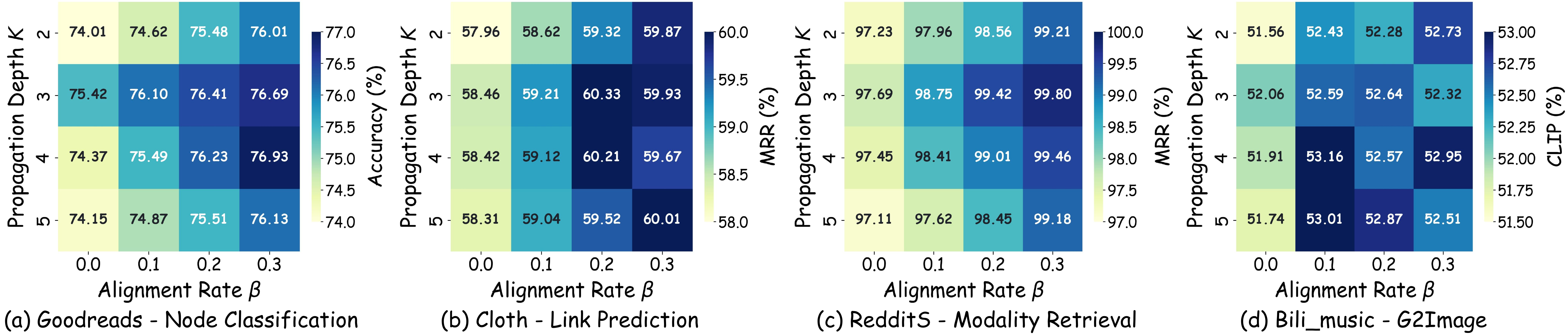}
\vspace{-0.5cm}
\caption{Hyperparameter sensitivity with respect to propagation depth $K$ and alignment rate $\beta$.
Each heatmap reports the task metric on one representative dataset-task pair.
CAMPA achieves stable performance across a broad range of moderate alignment rates and propagation depths.}
\label{fig:robustness}
\end{figure}

\section{Conclusion}

In this work, we propose CAMPA, a decoupled MAG framework that addresses modal conflict at both the propagation and aggregation stages.
Motivated by our empirical observations that decoupled MAG learning is scalable yet prone to insufficient cross-modal coordination, CAMPA introduces CAP to inject cross-modal semantic priors into graph diffusion and TAA to align multi-hop trajectories before hop fusion.
This design preserves the efficiency advantages of decoupled propagation while improving modality alignment and downstream representation quality.
Theoretical analysis further supports the stability of aligned propagation and the discrepancy-controlled behavior of trajectory fusion.
Extensive experiments across graph-centric and modality-centric tasks demonstrate that CAMPA consistently improves over coupled, decoupled, and recent multimodal graph baselines, while retaining favorable efficiency.
These results highlight stage-wise alignment as an effective principle for scalable multimodal graph learning.
Future work may extend CAMPA to richer modality combinations, stronger generative settings, and more adaptive large-scale MAG benchmarks.

\newpage

\bibliographystyle{plain}
\bibliography{reference}

%%%%%%%%%%%%%%%%%%%%%%%%%%%%%%%%%%%%%%%%%%%%%%%%%%%%%%%%%%%%

\newpage

\appendix

\section{Supplementary Empirical Analysis}
\label{app:empirical}

This appendix supplements the empirical studies in Sec.~\ref{sec:intro} and provides additional details on the efficiency comparison between coupled and decoupled architectures as well as the modal conflict phenomenon in decoupled multimodal graph learning.

\subsection{Setup of Empirical Studies}
\label{app:empirical_setup}

\textbf{Empirical Study (a): Coupled vs. Decoupled Architectures.}
We compare representative coupled baselines, namely MGAT and MMGCN, against representative decoupled baselines, including MSGC, SIGN, and MGDN, on the Toys dataset under the node classification setting.
MSGC is our in-house multimodal adaptation of SGC~\cite{wu2019sgc}.
Specifically, we extend the decoupled propagate-then-classify paradigm to the multimodal setting by performing modality-wise graph propagation on text and image features, followed by a shared prediction module for downstream classification.
For fairness, MSGC uses the same input features, training budget, and evaluation protocol as the other baselines.
The Toys dataset is extracted from the "Toys \& Games" category of Amazon2018~\cite{ni2019justifying}.
All methods use the same input modalities, train/validation/test split, and evaluation protocol, while their training time is recorded under the same hardware environment.
We report the accuracy-time trajectories throughout training to jointly examine predictive performance and computational efficiency.

\textbf{Empirical Study (b): Stage-wise Alignment Analysis.}
To investigate whether modal conflict is an intrinsic issue in decoupled pipelines, we augment representative decoupled baselines with lightweight alignment modules at different stages, including propagation-stage alignment (AlignP), aggregation-stage alignment (AlignA), and their combination.
Consistent improvements brought by these lightweight alignment modules would indicate that the corresponding stage indeed suffers from insufficient multimodal interaction.

\textbf{Propagation-stage alignment (AlignP).}
Let $\{\mathbf{H}_{t}^{(k)}\}_{k=0}^{K}$ and $\{\mathbf{H}_{i}^{(k)}\}_{k=0}^{K}$ denote the text and image trajectories produced by a generic decoupled baseline.
AlignP injects lightweight same-hop cross-modal correction during propagation:
\begin{equation}
\widehat{\mathbf{H}}_{t}^{(k)}=
(1-\lambda_{p})\mathbf{H}_{t}^{(k)}
+ \lambda_{p}\mathbf{P}_{i\rightarrow t}\mathbf{H}_{i}^{(k)},
\quad
\widehat{\mathbf{H}}_{i}^{(k)}=
(1-\lambda_{p})\mathbf{H}_{i}^{(k)}
+ \lambda_{p}\mathbf{P}_{t\rightarrow i}\mathbf{H}_{t}^{(k)},
\label{eq:empirical_alignp}
\end{equation}
where $\mathbf{P}_{i\rightarrow t}$ and $\mathbf{P}_{t\rightarrow i}$ are lightweight learnable projections used only to match the feature spaces, and $\lambda_{p}$ controls the strength of propagation-stage alignment.
This design preserves the decoupled structure of the baseline while introducing a minimal form of cross-modal interaction during multi-hop diffusion.

\textbf{Aggregation-stage alignment (AlignA).}
AlignA leaves the propagation stage unchanged and introduces coordination only when the baseline aggregates the multi-hop trajectories.
Let $\mathrm{Agg}(\cdot)$ denote the original hop aggregation operator of the decoupled baseline.
We first obtain modality-specific aggregated embeddings:
\begin{equation}
\mathbf{z}_{t}=\mathrm{Agg}\!\left(\{\mathbf{H}_{t}^{(k)}\}_{k=0}^{K}\right),
\quad
\mathbf{z}_{i}=\mathrm{Agg}\!\left(\{\mathbf{H}_{i}^{(k)}\}_{k=0}^{K}\right),
\label{eq:empirical_agga}
\end{equation}
and then apply a lightweight cross-modal refinement:
\begin{equation}
\widehat{\mathbf{z}}_{t}=
(1-\lambda_{a})\mathbf{z}_{t}
+ \lambda_{a}\mathbf{Q}_{i\rightarrow t}\mathbf{z}_{i},
\quad
\widehat{\mathbf{z}}_{i}=
(1-\lambda_{a})\mathbf{z}_{i}
+ \lambda_{a}\mathbf{Q}_{t\rightarrow i}\mathbf{z}_{t},
\label{eq:empirical_aligna}
\end{equation}
where $\mathbf{Q}_{i\rightarrow t}$ and $\mathbf{Q}_{t\rightarrow i}$ are lightweight learnable projections and $\lambda_{a}$ controls the strength of aggregation-stage alignment.
The final fused representation is then formed as $\widehat{\mathbf{z}}=\widehat{\mathbf{z}}_{t}\|\widehat{\mathbf{z}}_{i}$.
Compared with AlignP, this module does not alter the diffusion process and only improves the coordination of the final multimodal summaries.

\textbf{Combined alignment.}
For the \texttt{+AlignP+AlignA} setting, we first apply AlignP to obtain aligned trajectories $\{\widehat{\mathbf{H}}_{t}^{(k)}\}$ and $\{\widehat{\mathbf{H}}_{i}^{(k)}\}$, and then apply AlignA on top of the corresponding aggregated embeddings.
This setting examines whether propagation-stage and aggregation-stage coordination are complementary.

\textbf{Visualization and measurement.}
Beyond the quantitative gains shown in Fig.~\ref{fig:empirical}(b), we further provide heatmap-based correlation analysis and t-SNE visualizations in Appendix~\ref{app:empirical_vis} to illustrate how alignment affects cross-modal consistency and latent-space geometry.

\subsection{Visual Analysis of Modal Conflict}
\label{app:empirical_vis}

To further support the empirical study (b) in Sec.~\ref{sec:intro}, we provide additional evidence on the modal conflict phenomenon.
These visualizations show how the lack of cross-modal coordination leads to representation inconsistency, and how stage-wise alignment progressively alleviates this issue.

\textbf{Visualization protocol.}
We analyze representative decoupled baselines under four settings, including the original model, propagation-stage alignment (\texttt{+AlignP}), aggregation-stage alignment (\texttt{+AlignA}), and their combination (\texttt{+AlignP+AlignA}).
For each baseline, we report two complementary forms of visualization.
The first is a heatmap of pairwise feature correlations among image, text, and fused representations, which reflects the degree of cross-modal consistency in the learned space.
The second is a t-SNE visualization of image and text embeddings, which reveals whether the two modalities remain spatially separated or become progressively aligned after introducing stage-wise coordination.

\textbf{Heatmap analysis.}
Fig.~\ref{fig:appendix1} presents the heatmap-based comparison.
The top, middle, and bottom rows correspond to MSGC, MGDN, and SIGN, respectively.
For MSGC, the original model exhibits relatively weak Image-Text correlation, whereas \texttt{AlignP} substantially improves the correlation structure and \texttt{AlignP+AlignA} yields the strongest overall consistency.
The same trend can also be observed on MGDN and SIGN, where the original model starts from a modality discrepancy, and both alignment modules improve the correlation between modality-specific and fused representations.
Overall, the combined variant consistently produces the strongest cross-modal consistency, indicating that modal conflict exists in decoupled pipelines and can be reduced by stage-wise alignment.

\begin{figure*}[t]
\centering
\includegraphics[width=\textwidth]{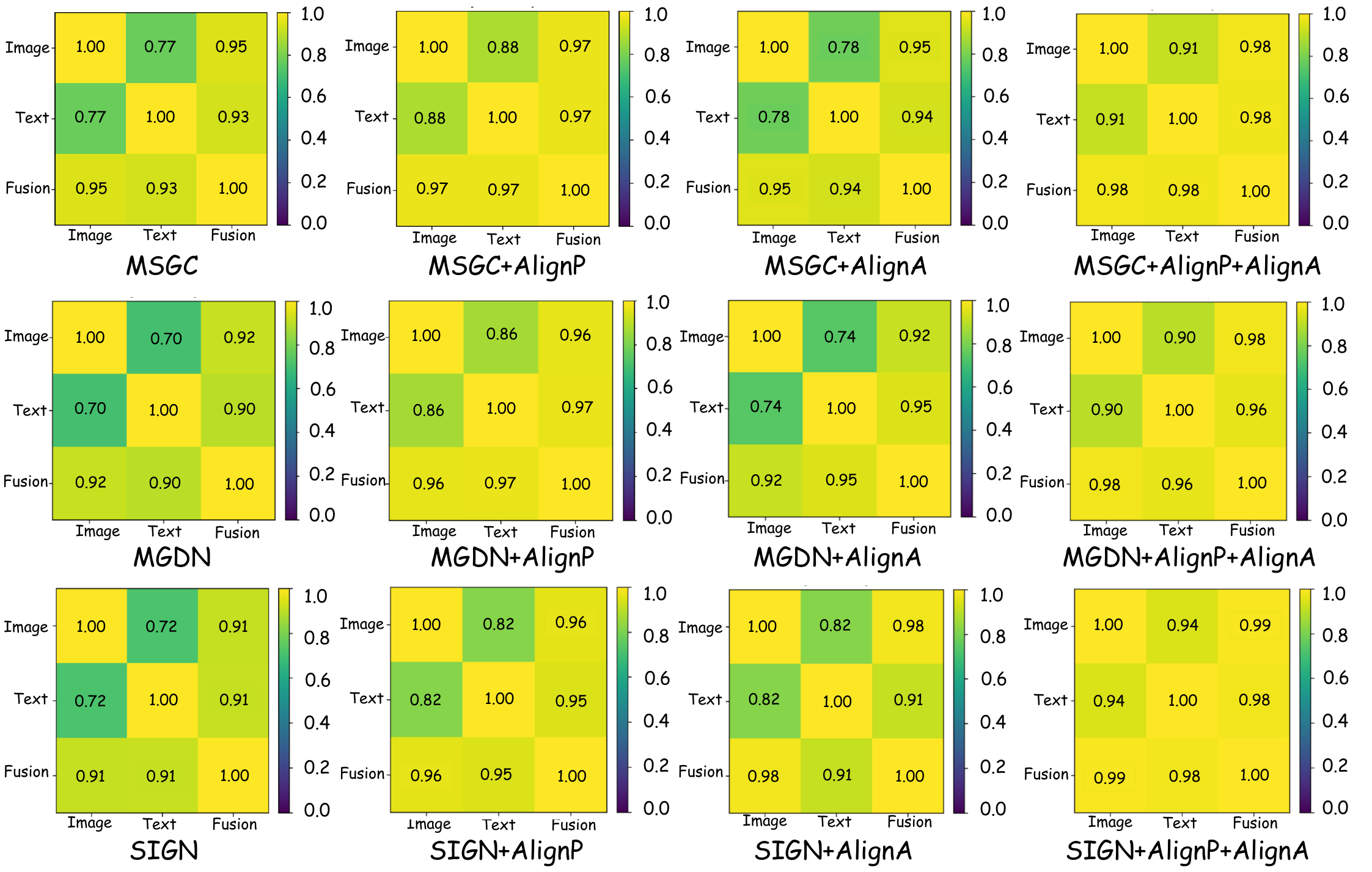}
\vspace{-0.5cm}
\caption{Heatmap visualization of cross-modal feature correlations. From top to bottom, we show the results of MSGC, MGDN, and SIGN. Within each row, (a) denotes the original decoupled baseline, (b) augments it with propagation-stage alignment, (c) augments it with aggregation-stage alignment, and (d) applies both alignment modules together.}
\label{fig:appendix1}
\end{figure*}

\textbf{T-SNE analysis.}
Fig.~\ref{fig:appendix2} further provides a geometric view of the same phenomenon.
For MSGC, MGDN, and SIGN, the original models produce clearly separated image and text clusters, suggesting that the two modalities evolve in different semantic spaces when propagated and aggregated independently.
After introducing \texttt{AlignP}, the two modality distributions become substantially more mixed, showing that propagation-stage coordination can directly reduce spatial asynchrony.
Adding \texttt{AlignA} also improves the overlap, while the full variant with \texttt{AlignP+AlignA} yields the most coherent multimodal geometry.
These observations further verify that the empirical gains in Fig.~\ref{fig:empirical}(b) originate from a genuine reduction of modal conflict rather than incidental optimization effects.

\begin{figure*}[t]
\centering
\includegraphics[width=\textwidth]{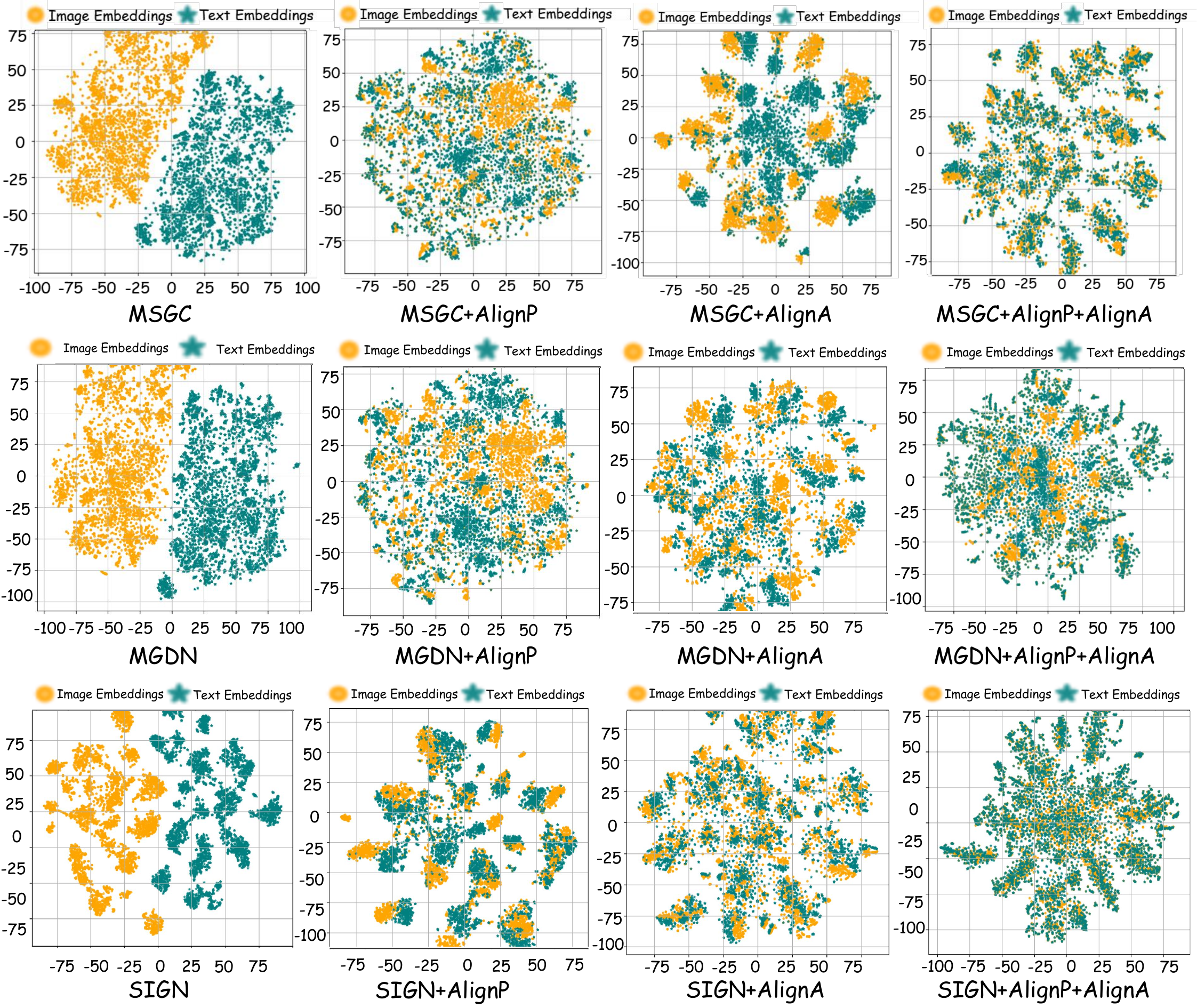}
\vspace{-0.5cm}
\caption{T-SNE visualization of image and text embeddings. From top to bottom, we show the results of MSGC, MGDN, and SIGN. The original decoupled models exhibit obvious modality separation, whereas stage-wise alignment progressively increases the overlap between image and text embeddings.}
\label{fig:appendix2}
\end{figure*}

\textbf{Discussion.}
Taken together, the heatmap and t-SNE results offer a more intuitive explanation of the empirical study in Sec.~\ref{sec:intro}.
They show that modal conflict in decoupled MGNNs is reflected not only in final performance degradation, but also in weak representation-level correlation and clear geometric separation between modalities.
More importantly, they also show that the two alignment modules play complementary roles: \texttt{AlignP} mainly improves cross-modal consistency during propagation, while \texttt{AlignA} further refines the coordination of multimodal trajectories during aggregation.
This provides direct motivation for the two-stage design of CAMPA.

\section{Proofs of Theoretical Analysis}
\label{app:theory}

This appendix provides detailed proofs for the theoretical results in Sec.~\ref{sec:theory}.
We first prove the convergence and boundedness of CAP, and then establish the bounded and discrepancy-controlled property of TAA.

\subsection{Proof of Theorem~\ref{thm:cap_stability}}

We restate the CAP recursion:
\begin{equation}
\mathbf{X}^{(k+1)}=\mathbf{P}\mathbf{X}^{(k)}+\mathbf{R}\mathbf{X}^{(0)},
\label{eq:app_cap_recursion}
\end{equation}
where $\mathbf{P}$ and $\mathbf{R}$ are defined in Eq.~\eqref{eq:theory_cap_block}.
For any block vector $\mathbf{X}=[\mathbf{X}_{t};\mathbf{X}_{i}]$, recall the norm
\[
\|\mathbf{X}\|_{2,\infty}=\max(\|\mathbf{X}_{t}\|_{F},\|\mathbf{X}_{i}\|_{F}).
\]

\begin{proof}
Since each $\widetilde{\mathbf{A}}_{uv}$ is symmetrically normalized, we have $\|\widetilde{\mathbf{A}}_{uv}\|_{2}\le 1$.
Therefore, for any $\mathbf{X}=[\mathbf{X}_{t};\mathbf{X}_{i}]$,
\begin{equation}
\mathbf{P}\mathbf{X}=
\begin{bmatrix}
\alpha_{t}\widetilde{\mathbf{A}}_{tt}\mathbf{X}_{t}
+ \beta_{t}\widetilde{\mathbf{A}}_{ti}\mathbf{X}_{i}\\
\beta_{i}\widetilde{\mathbf{A}}_{it}\mathbf{X}_{t}
+ \alpha_{i}\widetilde{\mathbf{A}}_{ii}\mathbf{X}_{i}
\end{bmatrix}.
\label{eq:app_cap_px}
\end{equation}
Applying the submultiplicativity of the Frobenius norm yields
\begin{align}
\|\alpha_{t}\widetilde{\mathbf{A}}_{tt}\mathbf{X}_{t}
+ \beta_{t}\widetilde{\mathbf{A}}_{ti}\mathbf{X}_{i}\|_{F}
&\le
\alpha_{t}\|\widetilde{\mathbf{A}}_{tt}\mathbf{X}_{t}\|_{F}
+ \beta_{t}\|\widetilde{\mathbf{A}}_{ti}\mathbf{X}_{i}\|_{F}\notag\\
&\le
\alpha_{t}\|\mathbf{X}_{t}\|_{F}
+ \beta_{t}\|\mathbf{X}_{i}\|_{F}\notag\\
&\le
(\alpha_{t}+\beta_{t})\|\mathbf{X}\|_{2,\infty},
\label{eq:app_cap_top}
\end{align}
and similarly,
\begin{equation}
\|\beta_{i}\widetilde{\mathbf{A}}_{it}\mathbf{X}_{t}
+ \alpha_{i}\widetilde{\mathbf{A}}_{ii}\mathbf{X}_{i}\|_{F}
\le
(\alpha_{i}+\beta_{i})\|\mathbf{X}\|_{2,\infty}.
\label{eq:app_cap_bottom}
\end{equation}
Combining Eqs.~\eqref{eq:app_cap_top} and \eqref{eq:app_cap_bottom}, we obtain
\begin{equation}
\|\mathbf{P}\mathbf{X}\|_{2,\infty}
\le
\rho \|\mathbf{X}\|_{2,\infty},
\quad
\rho=\max(\alpha_{t}+\beta_{t},\alpha_{i}+\beta_{i}).
\label{eq:app_cap_contraction}
\end{equation}
Hence, if $\rho<1$, $\mathbf{P}$ is a contraction under $\|\cdot\|_{2,\infty}$.

The fixed point $\mathbf{X}^{\star}$ of Eq.~\eqref{eq:app_cap_recursion} satisfies
\begin{equation}
\mathbf{X}^{\star}=\mathbf{P}\mathbf{X}^{\star}+\mathbf{R}\mathbf{X}^{(0)},
\label{eq:app_cap_fixed_1}
\end{equation}
which is equivalent to
\begin{equation}
(\mathbf{I}-\mathbf{P})\mathbf{X}^{\star}=\mathbf{R}\mathbf{X}^{(0)}.
\label{eq:app_cap_fixed_2}
\end{equation}
Since $\rho<1$, the matrix $\mathbf{I}-\mathbf{P}$ is invertible and thus the fixed point is unique:
\begin{equation}
\mathbf{X}^{\star}=(\mathbf{I}-\mathbf{P})^{-1}\mathbf{R}\mathbf{X}^{(0)}.
\label{eq:app_cap_fixed_3}
\end{equation}

Subtracting Eq.~\eqref{eq:app_cap_fixed_1} from Eq.~\eqref{eq:app_cap_recursion} gives
\begin{equation}
\mathbf{X}^{(k+1)}-\mathbf{X}^{\star}
=
\mathbf{P}\bigl(\mathbf{X}^{(k)}-\mathbf{X}^{\star}\bigr).
\label{eq:app_cap_error}
\end{equation}
Taking the block norm on both sides and using Eq.~\eqref{eq:app_cap_contraction}, we obtain
\begin{equation}
\|\mathbf{X}^{(k+1)}-\mathbf{X}^{\star}\|_{2,\infty}
\le
\rho \|\mathbf{X}^{(k)}-\mathbf{X}^{\star}\|_{2,\infty},
\label{eq:app_cap_linear}
\end{equation}
which proves the linear convergence claim in Theorem~\ref{thm:cap_stability}.

For completeness, we further bound the magnitude of the fixed point.
From Eq.~\eqref{eq:app_cap_fixed_1},
\begin{equation}
\|\mathbf{X}^{\star}\|_{2,\infty}
\le
\|\mathbf{P}\mathbf{X}^{\star}\|_{2,\infty}
+ \|\mathbf{R}\mathbf{X}^{(0)}\|_{2,\infty}
\le
\rho \|\mathbf{X}^{\star}\|_{2,\infty}
+ \max(\gamma_{t},\gamma_{i})\|\mathbf{X}^{(0)}\|_{2,\infty}.
\label{eq:app_cap_bound_1}
\end{equation}
Rearranging yields
\begin{equation}
\|\mathbf{X}^{\star}\|_{2,\infty}
\le
\frac{\max(\gamma_{t},\gamma_{i})}{1-\rho}\|\mathbf{X}^{(0)}\|_{2,\infty}.
\label{eq:app_cap_bound_2}
\end{equation}
This completes the proof.
\end{proof}

\subsection{Proof of Proposition~\ref{prop:taa_bound}}

For one node, recall that
\begin{equation}
\mathbf{z}_{u}=\sum_{k=0}^{K} a_{u}^{(k)}\mathbf{g}_{u}^{(k)}, \quad u\in\{t,i\},
\label{eq:app_taa_node}
\end{equation}
where $a_{u}^{(k)}\ge 0$ and $\sum_{k=0}^{K}a_{u}^{(k)}=1$.

\begin{proof}
Because $\{a_{u}^{(k)}\}_{k=0}^{K}$ are nonnegative and sum to one, $\mathbf{z}_{u}$ is a convex combination of $\{\mathbf{g}_{u}^{(0)},\ldots,\mathbf{g}_{u}^{(K)}\}$.
Therefore, $\mathbf{z}_{u}\in \operatorname{conv}\{\mathbf{g}_{u}^{(0)},\ldots,\mathbf{g}_{u}^{(K)}\}$.
Applying the triangle inequality gives
\begin{equation}
\|\mathbf{z}_{u}\|_{2}
=
\left\|\sum_{k=0}^{K} a_{u}^{(k)}\mathbf{g}_{u}^{(k)}\right\|_{2}
\le
\sum_{k=0}^{K} a_{u}^{(k)}\|\mathbf{g}_{u}^{(k)}\|_{2}
\le
\max_{0\le k\le K}\|\mathbf{g}_{u}^{(k)}\|_{2},
\label{eq:app_taa_norm}
\end{equation}
which establishes the boundedness of $\mathbf{z}_{u}$ as a convex combination of hop-wise representations.

Next, consider the difference between the two fused representations:
\begin{align}
\mathbf{z}_{t}-\mathbf{z}_{i}
&=
\sum_{k=0}^{K} a_{t}^{(k)}\mathbf{g}_{t}^{(k)}
- \sum_{k=0}^{K} a_{i}^{(k)}\mathbf{g}_{i}^{(k)}\notag\\
&=
\sum_{k=0}^{K} a_{t}^{(k)}\bigl(\mathbf{g}_{t}^{(k)}-\mathbf{g}_{i}^{(k)}\bigr)
+ \sum_{k=0}^{K}\bigl(a_{t}^{(k)}-a_{i}^{(k)}\bigr)\mathbf{g}_{i}^{(k)}.
\label{eq:app_taa_decomp}
\end{align}
Taking the Euclidean norm and applying the triangle inequality yields
\begin{align}
\|\mathbf{z}_{t}-\mathbf{z}_{i}\|_{2}
&\le
\left\|\sum_{k=0}^{K} a_{t}^{(k)}\bigl(\mathbf{g}_{t}^{(k)}-\mathbf{g}_{i}^{(k)}\bigr)\right\|_{2}
+ \left\|\sum_{k=0}^{K}\bigl(a_{t}^{(k)}-a_{i}^{(k)}\bigr)\mathbf{g}_{i}^{(k)}\right\|_{2}\notag\\
&\le
\sum_{k=0}^{K} a_{t}^{(k)}\|\mathbf{g}_{t}^{(k)}-\mathbf{g}_{i}^{(k)}\|_{2}
+ \sum_{k=0}^{K}|a_{t}^{(k)}-a_{i}^{(k)}|\,\|\mathbf{g}_{i}^{(k)}\|_{2},
\label{eq:app_taa_gap}
\end{align}
which is exactly Eq.~\eqref{eq:theory_taa_gap}.
The symmetric statement after swapping $t$ and $i$ follows from the same derivation.
\end{proof}

\section{Algorithm and Complexity Analysis}
\label{app:algo_complexity}

This appendix summarizes the training pipeline of CAMPA and analyzes its computational complexity.
Our goal is to make explicit which operations are precomputed in the aligned propagation stage and which operations remain in the online optimization stage.

\subsection{Pseudo-code of CAMPA}

\begin{algorithm}[t]
\caption{Training Pipeline of CAMPA}
\label{alg:campa}
\begin{algorithmic}[1]
\REQUIRE Graph adjacency $\mathbf{A}$, raw textual contents $\mathcal{C}_{t}$, raw visual contents $\mathcal{C}_{i}$, propagation depth $K$, balance coefficients $\alpha,\beta,\gamma$, task supervision
\ENSURE Trained TAA parameters and prediction head
\STATE Encode $\mathcal{C}_{t}$ and $\mathcal{C}_{i}$ with a frozen pre-trained vision-language encoder to obtain initial modality features $\mathbf{H}_{t}^{(0)}$ and $\mathbf{H}_{i}^{(0)}$.
\STATE \textbf{Offline aligned propagation}
\STATE Compute edge-wise semantic priors $\mathbf{S}_{tt}, \mathbf{S}_{ii}, \mathbf{S}_{ti}, \mathbf{S}_{it}$ and construct $\widetilde{\mathbf{A}}_{tt}, \widetilde{\mathbf{A}}_{ii}, \widetilde{\mathbf{A}}_{ti}, \widetilde{\mathbf{A}}_{it}$ by Eq.~\eqref{eq:cap_prior}.
\STATE Initialize the trajectories with $\mathbf{H}_{t}^{(0)}$ and $\mathbf{H}_{i}^{(0)}$.
\FOR{$k=0,\ldots,K-1$}
\STATE Update $\mathbf{H}_{t}^{(k+1)}$ and $\mathbf{H}_{i}^{(k+1)}$ by Eq.~\eqref{eq:cap_update}.
\ENDFOR
\STATE Stack the propagated slices to obtain $\mathbf{H}_{t}=[\mathbf{H}_{t}^{(0)},\ldots,\mathbf{H}_{t}^{(K)}]$ and $\mathbf{H}_{i}=[\mathbf{H}_{i}^{(0)},\ldots,\mathbf{H}_{i}^{(K)}]$.
\STATE \textbf{Online trajectory-aligned training}
\FOR{each training epoch}
\STATE Perform intra-modal trajectory encoding on $\mathbf{H}_{t}$ and $\mathbf{H}_{i}$ to obtain $\mathbf{F}_{t}$ and $\mathbf{F}_{i}$ by Eq.~\eqref{eq:taa_self}.
\STATE Apply inter-modal trajectory interaction to obtain $\mathbf{G}_{t}$ and $\mathbf{G}_{i}$ by Eq.~\eqref{eq:taa_cross}.
\STATE Compute hop-wise consistency $\mathbf{C}$ and modality-specific hop weights $\mathbf{A}_{t}, \mathbf{A}_{i}$ by Eqs.~\eqref{eq:taa_consistency} and \eqref{eq:taa_gate}.
\STATE Aggregate $\mathbf{Z}_{t}$ and $\mathbf{Z}_{i}$ and form $\mathbf{Z}=\mathbf{Z}_{t}\|\mathbf{Z}_{i}$ by Eq.~\eqref{eq:taa_output}.
\STATE Optimize the TAA parameters and task head with the downstream loss.
\ENDFOR
\STATE \textbf{Return} the trained CAMPA model.
\end{algorithmic}
\end{algorithm}

Algorithm~\ref{alg:campa} highlights the decoupled nature of CAMPA.
The graph-dependent aligned diffusion in CAP is executed once before training, whereas the online optimization stage only operates on the precomputed hop trajectories.
This separation is the key reason why CAMPA retains the efficiency advantage of decoupled multimodal graph learning.

\subsection{Complexity Analysis}

Let $N$ be the number of nodes, $|E|$ the number of observed edges, $d$ the hidden dimension, $K$ the propagation depth, and $B$ the batch size used in the online training stage.
For simplicity, we omit constant factors introduced by the two modalities and the four aligned propagation operators.

\textbf{CAP precomputation.}
To construct the aligned operators in Eq.~\eqref{eq:cap_prior}, CAMPA only needs semantic similarities on observed edges because the prior is multiplied by the sparse adjacency matrix $\mathbf{A}$.
Computing the edge-wise semantic priors therefore costs $\mathcal{O}(|E|d)$, while symmetric normalization of the resulting sparse operators costs $\mathcal{O}(|E|)$.
The subsequent $K$-hop aligned diffusion in Eq.~\eqref{eq:cap_update} requires $K$ sparse matrix-feature multiplications, leading to a total CAP precomputation cost of $\mathcal{O}(K|E|d)$.
Storing the propagated trajectories for both modalities requires $\mathcal{O}(N(K+1)d)$ memory.

\textbf{Online TAA training.}
Once the trajectories are precomputed, the online training stage no longer depends on the graph edges.
For a batch of $B$ nodes, the linear projections and MLP blocks in TAA cost $\mathcal{O}(B(K+1)d^{2})$, while self-attention and cross-attention along the hop dimension cost $\mathcal{O}(B(K+1)^{2}d)$.
The consistency scoring, gating, and weighted hop aggregation introduce an additional $\mathcal{O}(B(K+1)d)$ cost.
Therefore, the per-epoch online complexity of CAMPA can be summarized as $\mathcal{O}(B(K+1)d^{2} + B(K+1)^{2}d)$.
Since $K$ is typically small and fixed, the online training cost is independent of $|E|$ and scales only with the trajectory length and hidden dimension.

\textbf{Overall efficiency advantage.}
Overall, the total cost of CAMPA consists of a one-time graph-dependent precomputation term and an edge-independent online optimization term.
This is fundamentally different from coupled message-passing MGNNs, whose per-epoch training cost usually scales with repeated neighborhood aggregation over edges.
By moving aligned diffusion to the precomputation stage, CAMPA preserves the efficiency benefits of decoupled training while still introducing explicit cross-modal coordination in both propagation and aggregation.

\section{Dataset Description}
\label{app:datasets}

We evaluate CAMPA on eight multimodal graph datasets spanning recommendation, social, media, music, art, and book domains.
Table~\ref{tab:dataset_stats_app} summarizes their modality composition, graph scale, label cardinality, representative task scope, and application domain used in our experiments.

\begin{table*}[h]
\centering
\caption{Statistics of the eight MAG datasets used in our experiments.}
\label{tab:dataset_stats_app}
\small
\setlength{\tabcolsep}{5pt}
\resizebox{\textwidth}{!}{
\begin{tabular}{lcccccc}
\toprule
Dataset & Modalities & \#Nodes & \#Edges & \#Classes & Task Scope & Domain \\
\midrule
\texttt{Ele-fashion} & Text + Image & 97,766  & 199,602   & 12 & Graph + Modality & Recommendation \\
\texttt{Bili\_music} & Text + Image & 6,038   & 21,592    & - & Graph + Modality & Video Recommendation \\
\texttt{RedditS}     & Text + Image & 15,894  & 566,160   & 20 & Graph + Modality & Social Network \\
\texttt{Movies}      & Text + Image & 16,672  & 218,390   & 20 & Graph + Modality & Movie Network \\
\texttt{Goodreads}   & Text + Image & 685,294 & 7,235,048 & 11 & Graph            & Book Network \\
\texttt{Grocery}     & Text + Image & 17,074  & 171,340   & 20 & Graph + Modality & Recommendation \\
\texttt{Cloth}       & Text + Image & 125,839 & 951,271   & - & Graph + Modality & Recommendation \\
\texttt{SemArt}      & Text + Image & 21,382  & 1,216,432 & - & Graph + Modality & Art Network \\
\bottomrule
\end{tabular}
}
\end{table*}

Below we briefly describe each dataset used in our benchmark.

\textbf{\texttt{Ele-fashion}}~\cite{ni2019_Grocery_Cloth_Ele_Movies_Sports,hou2024_Cloth_Ele_Sports}
is a heterogeneous product graph that combines Amazon electronics and fashion items.
Nodes represent products and edges capture cross-category co-purchasing relations, revealing latent consumer preferences across distinct domains.
Each node is associated with textual descriptions and product images, and the dataset supports both graph-centric and modality-centric evaluation.

\textbf{\texttt{Bili\_music}}~\cite{wan2026openmag}
is a music-content multimodal graph constructed from video or music-related entities.
Its graph structure reflects content association and user-consumption correlation, while node attributes combine textual metadata with visual content.
We mainly use this dataset to evaluate graph-aware multimodal understanding and retrieval behavior in content-oriented scenarios.

\textbf{\texttt{RedditS}}~\cite{desai1_RedditS}
is a social multimodal graph in which nodes correspond to Reddit posts and edges encode interaction or reply relations.
The textual modality is derived from post titles and body content, while the visual modality comes from attached images.
This dataset is particularly suitable for evaluating socially grounded node classification, clustering, and link prediction.

\textbf{\texttt{Movies}}~\cite{ni2019_Grocery_Cloth_Ele_Movies_Sports}
is a recommendation-style multimodal graph built on movie or media entities.
Nodes contain textual descriptions and poster-like visual signals, while edges reflect semantic or behavioral relations among items.
It provides a representative benchmark for studying multimodal graph learning in media recommendation settings.

\textbf{\texttt{Goodreads}}~\cite{wan2018_Goodreads_NC,wan2019_Goodreads_NC}
is a large-scale book graph where nodes represent books and edges encode user-behavior or shelving relations.
Its textual modality captures summaries and editorial descriptions, whereas the visual modality is derived from book covers.
The scale of this dataset makes it particularly useful for evaluating the scalability of efficient decoupled MAG methods.

\textbf{\texttt{Grocery}}~\cite{ni2019_Grocery_Cloth_Ele_Movies_Sports}
is a product graph from the recommendation domain, where edges describe complementary or co-purchasing patterns.
Textual features reflect product titles and descriptions, while visual features come from product appearance or packaging images.
This dataset is commonly used to evaluate cross-modal recommendation and alignment quality.

\textbf{\texttt{Cloth}}~\cite{ni2019_Grocery_Cloth_Ele_Movies_Sports,hou2024_Cloth_Ele_Sports}
is a fashion-oriented recommendation graph that emphasizes visually sensitive item relations.
Nodes contain both textual descriptions and clothing images, and edges reflect compatibility or co-consumption signals.
The dataset is well suited for testing whether multimodal graph models can jointly capture semantic and visual compatibility.

\textbf{\texttt{SemArt}}~\cite{garcia2018_SemArt}
is an art-understanding benchmark where nodes correspond to artworks and edges encode semantic or metadata-based relations.
Each node combines textual commentary with artwork images, providing a challenging testbed for graph-conditioned multimodal generation and alignment.
We include it as a representative modality-centric benchmark, especially for graph-to-image related evaluation.

\section{Baseline Details}
\label{app:baselines}

\textbf{GAT}~\cite{velivckovic2017gat}
is a classical attention-based GNN that performs masked neighborhood aggregation with learnable edge attention coefficients.
It can be viewed as a representative extension of early message-passing GNNs from convolution-style neighborhood smoothing to adaptive attention-based propagation~\cite{kipf2016gcn,brody2021gatv2}.
Although it is not designed for MAG learning, we include it as a general graph reference to calibrate the benefit of injecting multimodal information into graph representation learning.
In our implementation, GAT receives the same graph structure and encoder-derived node features as the multimodal methods, so that the comparison isolates the gain from multimodal modeling rather than differences in raw inputs.

\textbf{MMGCN}~\cite{MMGCN}
is one of the early representative multimodal graph models for recommendation.
Its core design is to construct modality-specific graph encoders and then combine their outputs under shared graph supervision, thereby capturing complementary preference signals from different modalities.
We include MMGCN because it represents a typical coupled MAG pipeline, in which message passing and multimodal fusion are tightly interleaved throughout training.
This makes it a useful reference for assessing whether CAMPA can preserve the strengths of multimodal graph modeling while avoiding the efficiency limitations of coupled propagation.

\textbf{MGAT}~\cite{MGAT}
extends graph attention to multimodal recommendation by introducing gated attention mechanisms over graph neighbors and multimodal channels.
Compared with MMGCN, MGAT places more emphasis on adaptive modality weighting during propagation and thus serves as another representative coupled MAG baseline with stronger interaction modeling capacity.
We compare against MGAT to test whether CAMPA can still maintain an advantage when the baseline already performs nontrivial attention-based fusion inside the graph learning process.

\textbf{DMGC}~\cite{DMGC}
is a recent multimodal graph clustering method that explicitly disentangles homophily-oriented and heterophily-oriented graph patterns.
It combines dual-branch structural modeling with multimodal fusion, aiming to improve clustering quality under complex graph semantics.
We include DMGC because it is a strong modern baseline for graph-centric unsupervised MAG tasks and represents a design philosophy different from recommendation-oriented models.
Its inclusion helps verify that CAMPA is not only competitive on recommendation benchmarks, but also remains effective on general representation learning scenarios.

\textbf{DGF}~\cite{DGF}
adopts dual graph filtering together with cross-modal contrastive objectives to enhance robustness against noisy graph structure and modality corruption.
Instead of relying on a single propagation stream, DGF uses complementary filtered views to preserve informative low-frequency and high-frequency components.
We include it as a recent strong baseline because it explicitly addresses noisy multimodal graph signals, which overlaps with part of the motivation behind our alignment-based design.
This comparison is therefore important for showing that the improvements of CAMPA are not simply due to stronger denoising, but come from better stage-wise coordination across modalities.

\textbf{MIG-GT}~\cite{MIG-GT}
combines modality-independent graph encoders with a global transformer, thereby coupling local graph aggregation and long-range multimodal interaction in a unified framework.
Its main appeal lies in balancing graph inductive bias with global dependency modeling, which makes it a strong reference for modern MAG architectures beyond purely local message passing.
We compare against MIG-GT because it reflects a powerful coupled paradigm where multimodal coordination is handled through deep joint modeling rather than decoupled propagation.

\textbf{LGMRec}~\cite{LGMRec}
is a representative lightweight multimodal recommendation model that integrates graph topology with modality-aware preference learning.
It is particularly competitive in recommendation settings because it can effectively exploit user-item structural signals together with textual and visual content while maintaining relatively modest model complexity.
We include LGMRec as a practical and task-oriented baseline, since it reflects the performance level of modern multimodal recommendation systems rather than purely generic MAG encoders.

\textbf{UniGraph2}~\cite{he2025unigraph2}
learns a unified graph-aware embedding space that binds multiple modalities into a shared representation.
Compared with earlier recommendation-focused methods, UniGraph2 emphasizes more general multimodal graph representation learning and stronger cross-modal semantic coupling.
We treat it as a recent strong baseline because its unified embedding design is conceptually close to the broader goal of reducing modality discrepancy, making it a particularly meaningful comparator for CAMPA.

\textbf{SIGN}~\cite{frasca2020sign}
is a classical decoupled GNN that precomputes multi-hop propagated features and performs downstream learning using lightweight predictors.
Its key advantage is that graph diffusion is separated from trainable transformation, substantially reducing training cost on large graphs.
Although SIGN itself is unimodal, it provides the canonical propagate-then-predict template for efficient decoupled graph learning.
We therefore include it to evaluate whether the efficiency benefits of decoupling can be carried over to MAG scenarios without sacrificing multimodal modeling quality.

\textbf{MSGC}~\cite{wu2019sgc}
denotes our multimodal adaptation of the SGC-style decoupled pipeline.
Concretely, each modality is propagated independently over the graph to obtain multi-hop representations, which are then fused only at the prediction stage.
This baseline is particularly important in our study because it strips decoupled MAG learning down to its simplest form: efficient independent propagation followed by shallow multimodal aggregation.
As a result, MSGC serves as a direct probe of the modal conflict issue discussed in the introduction and appendix empirical analysis.

\textbf{MGDN}~\cite{hu2024mgdcf}
is a recent decoupled multimodal method that leverages graph diffusion to produce efficient multimodal representations.
Compared with simpler decoupled baselines such as SIGN or MSGC, MGDN introduces a stronger multimodal design while retaining the core efficiency advantage of precomputed propagation.
We include MGDN as an especially relevant baseline because it is close to CAMPA along the efficiency axis, allowing us to test whether our gains truly come from aligned propagation and aggregation rather than from decoupling alone.

\section{Evaluation Protocols}
\label{app:eval_protocol}

To ensure a fair comparison across all methods, we evaluate every model under the same dataset split, graph topology, and modality inputs for each dataset-task pair.
Unless otherwise specified, larger values indicate better performance.

\textbf{Node Classification.}
Node classification is the basic supervised graph task.
Given a MAG, each method encodes node $v$ into a low-dimensional representation $\mathbf{z}_v$, which is then fed into a projection head and a softmax classifier.
Optimization is performed by minimizing the cross-entropy loss on the labeled training nodes:
\begin{equation}
\mathcal{L}_{\mathrm{nc}} = - \sum_{v \in \mathcal{V}_{\mathrm{tr}}} \sum_{c=1}^{C} y_{v,c} \log \hat{y}_{v,c},
\end{equation}
where $C$ is the number of classes, $y_{v,c}$ is the one-hot ground-truth label, and $\hat{y}_{v,c}$ is the predicted probability.
We report Accuracy and F1-score on the test split.

\textbf{Node Clustering.}
Node clustering evaluates the quality of learned node embeddings in an unsupervised setting.
For each method, we first extract the final node representations and then apply the same clustering backend to produce semantic partitions.
Following standard graph clustering evaluation, we report ACC, NMI, and ARI.
Among them, NMI measures the mutual dependence between predicted clusters $C$ and ground-truth labels $Y$:
\begin{equation}
\mathrm{NMI}(Y, C) = \frac{2 I(Y, C)}{H(Y) + H(C)},
\end{equation}
where $I(\cdot,\cdot)$ and $H(\cdot)$ denote mutual information and entropy, respectively.
ARI evaluates clustering consistency from the perspective of sample pairs:
\begin{equation}
\mathrm{ARI} =
\frac{\sum_{ij} \binom{n_{ij}}{2} - \left[\sum_i \binom{a_i}{2} \sum_j \binom{b_j}{2}\right] / \binom{n}{2}}
{\frac{1}{2}\left[\sum_i \binom{a_i}{2} + \sum_j \binom{b_j}{2}\right] - \left[\sum_i \binom{a_i}{2} \sum_j \binom{b_j}{2}\right] / \binom{n}{2}},
\end{equation}
where $n_{ij}$ denotes the overlap between ground-truth cluster $i$ and predicted cluster $j$.

\textbf{Link Prediction.}
Link prediction evaluates whether the learned representations can identify missing or potential edges.
For each query edge, the model computes scores for candidate node pairs and ranks positive edges above negative samples.
We report ranking-based metrics including MRR, Hits@3, and Hits@10.
Specifically, MRR is defined as
\begin{equation}
\mathrm{MRR} = \frac{1}{|\mathcal{Q}|} \sum_{q \in \mathcal{Q}} \frac{1}{\mathrm{rank}_{q}},
\end{equation}
where $\mathrm{rank}_{q}$ is the rank position of the first positive target for query $q$.
Hits@K measures the fraction of queries whose positive target is ranked within the top-$K$ candidates:
\begin{equation}
\mathrm{Hits@K} = \frac{1}{|\mathcal{Q}|} \sum_{q \in \mathcal{Q}} \mathbb{I}(\mathrm{rank}_{q} \le K).
\end{equation}

\textbf{Modality Retrieval.}
Modality retrieval evaluates whether graph-aware representations improve bidirectional cross-modal search in a shared latent space.
Given a source query from one modality, the model ranks candidates from the other modality by similarity.
We report T2I-MRR and I2T-MRR, following the same ranking definition as MRR in link prediction.

\textbf{Modality Matching.}
Modality matching measures whether paired textual and visual contents remain semantically consistent after graph-aware representation learning.
We follow the CLIP-based evaluation protocol~\cite{radford2021_clip} and report the average CLIP-score on matched text-image pairs.

\textbf{Modality Alignment.}
Modality alignment evaluates global semantic agreement between the textual and visual representations learned on the graph.
Following the same CLIP-based protocol~\cite{radford2021_clip}, we report the average alignment score between the two modalities.
This task is particularly relevant to CAMPA because it directly reflects whether stage-wise alignment reduces cross-modal discrepancy in MAG learning.

\textbf{G2Image.}
G2Image focuses on synthesizing images conditioned on graph-aware multimodal context, requiring the generated results to remain consistent with both the textual description and the graph structure.
Following graph-conditioned generation protocols~\cite{yoon2023mmgl,jin2024instructg2i}, we evaluate the generated images with CLIP-Score and DINOv2-Score.
Specifically, CLIP-Score measures cross-modal semantic consistency:
\begin{equation}
\mathrm{CLIP\mbox{-}Score}(I, T) = \max \left(100 \cdot \cos(E_I(I), E_T(T)), 0 \right),
\end{equation}
where $E_T$ and $E_I$ denote the pre-trained CLIP text and image encoders~\cite{radford2021_clip}.
DINOv2-Score evaluates visual fidelity and structural consistency against the reference image:
\begin{equation}
\mathrm{DINOv2\mbox{-}Score}(I_{\mathrm{gen}}, I_{\mathrm{ref}}) =
\cos \big( \mathrm{DINO}(I_{\mathrm{gen}}), \mathrm{DINO}(I_{\mathrm{ref}}) \big),
\end{equation}
where $\mathrm{DINO}(\cdot)$ denotes the DINOv2 encoder~\cite{oquab2024dinov2}. Together, these two metrics reflect semantic faithfulness and perceptual consistency of graph-conditioned generation.

\textbf{Repeated Runs and Reporting.}
When a task involves stochastic optimization, negative sampling, or clustering initialization, we repeat the evaluation under multiple random seeds and report the mean performance.
The variability terms reported in tables are standard deviations computed over these repeated runs.
All compared methods are evaluated with the same protocol on each dataset so that the reported differences can be attributed to the model design rather than to inconsistencies in splits, metrics, or evaluation backends.

\section{Experimental Settings}
\label{app:exp_settings}

\textbf{Shared Inputs and Preprocessing.}
All compared methods are evaluated on the same graph topology, train-validation-test split, and node-aligned multimodal contents for each dataset.
The raw inputs consist of the graph structure together with textual and visual contents associated with each node.
Before graph learning, the text and image contents are encoded by the same frozen pre-trained vision-language encoder, instantiated as Qwen2-VL-7B-Instruct~\cite{wang2024qwen2vlenhancingvisionlanguagemodels}, and the resulting embeddings are used as the initial modality features for all methods.
If a baseline requires a task-specific hidden dimension, we apply the corresponding trainable projection only after this shared feature extraction stage, so that the fairness of the input protocol is preserved.

\textbf{Algorithm Hyperparameters.}
For CAMPA, the propagation depth $K$ is searched within $[2, 3, 4, 5]$.
The hidden dimension is tuned over $[128, 256, 512]$.
The number of multi-head attention heads is searched within $[2, 4, 8]$.
The propagation coefficients are selected from $\alpha \in \{0.4, 0.5, 0.6\}$ and $\beta \in \{0.2, 0.3, 0.4\}$, while $\gamma$ is determined by the constraint $\alpha + \beta + \gamma = 1$.
The dropout rate is tuned in $[0.3, 0.4, 0.5]$.
For baseline methods, we adopt the hyperparameter configurations reported in their original papers whenever available.
For unspecified settings, we perform automated hyperparameter optimization using the Optuna framework.

\textbf{Task Hyperparameters.}
For Node Classification and Node Clustering, the learning rate is searched within $[1\times10^{-5}, 5\times10^{-5}, 1\times10^{-4}, 1\times10^{-3}, 5\times10^{-3}]$, and the weight decay is tuned within $[1\times10^{-5}, 1\times10^{-4}]$.
We fix the batch size to 512, the number of sampled neighbors to 25, and train the models for up to 300 epochs with an early-stop patience of 40.
For Link Prediction, we set a fixed learning rate of $1\times10^{-3}$, batch size of 1024, and weight decay of $1\times10^{-5}$.
The number of sampled neighbors is set to 15, the model is trained for up to 100 epochs, and early-stop patience is set to 20.
For Modality Retrieval, Modality Matching, and Modality Alignment, we adopt task-specific downstream contrastive models and objectives with a temperature scaling factor $\tau=0.1$.
The models are trained for up to 500 epochs with a learning rate of $1\times10^{-3}$, batch size of 256, and early-stop patience of 50.
For the G2Image~\cite{jin2024instructg2i} task, we adopt Stable Diffusion v1.5 as the image generation backbone.
The model is trained with a learning rate of $1\times10^{-4}$, batch size of 16, and 10 training epochs.
The input image resolution is set to 256, and graph-aware conditioning is injected with a cross-attention frequency of 2.
We use 4 semantic neighbors for graph-conditioned generation.
The optimization is performed with AdamW ($\beta_1=0.9$, $\beta_2=0.999$), weight decay of 0.01, and maximum gradient norm of 1.0.
A constant learning-rate scheduler with 500 warmup steps is used.
During inference, 50 denoising steps are adopted.

\textbf{Experimental Environment.}
Experiments are conducted on a workstation equipped with an Intel(R) Xeon(R) Platinum 8470Q processor and an NVIDIA RTX PRO 6000 Blackwell Server Edition GPU with 96GB of VRAM, supported by 220GB of system RAM.
The computational environment utilizes CUDA 13.1, while software implementations are developed using Python 3.10.18 and PyTorch 2.8.0.

\begin{table*}[!t]
\centering
\caption{Additional graph-centric results on representative datasets.}
\label{tab:appendix_graph_comp}
\footnotesize
\renewcommand{\arraystretch}{1.12}
\resizebox{\textwidth}{!}{
\setlength{\tabcolsep}{2.0mm}{
\begin{tabular}{l!{\vrule width 0.1pt}cc!{\vrule width 0.1pt}cc!{\vrule width 0.1pt}cc!{\vrule width 0.1pt}cc}
\Xhline{1pt}
\rowcolor{teal!72!black}
\multicolumn{1}{c!{\vrule width 0.1pt}}{\textbf{\textcolor{white}{Tasks}}} & \multicolumn{4}{c!{\vrule width 0.1pt}}{\textbf{\textcolor{white}{Node Classification}}} & \multicolumn{4}{c}{\textbf{\textcolor{white}{Node Clustering}}} \\
\hline
\rowcolor{teal!12}
 & \multicolumn{2}{c!{\vrule width 0.1pt}}{\textbf{RedditS}} & \multicolumn{2}{c!{\vrule width 0.1pt}}{\textbf{Movies}} & \multicolumn{2}{c!{\vrule width 0.1pt}}{\textbf{Ele-fashion}} & \multicolumn{2}{c}{\textbf{RedditS}} \\
\cline{2-9}
\rowcolor{teal!12}
\multirow{-2}{*}{\diagbox[width=8.2em,height=2.3em]{\textbf{Methods}}{\textbf{Datasets}}}
 & Acc & F1 & Acc & F1 & NMI & ARI & NMI & ARI \\
\hline
GAT & 94.21$_{\pm 0.13}$ & 90.23$_{\pm 0.13}$ & 53.98$_{\pm 0.24}$ & 43.57$_{\pm 0.24}$ & 35.45$_{\pm 0.30}$ & 17.04$_{\pm 0.30}$ & 37.40$_{\pm 0.30}$ & 17.78$_{\pm 0.30}$ \\
\hline
\rowcolor{teal!5} MMGCN & 91.98$_{\pm 0.13}$ & 85.36$_{\pm 0.13}$ & 58.36$_{\pm 0.25}$ & 47.31$_{\pm 0.25}$ & 39.84$_{\pm 0.31}$ & 38.31$_{\pm 0.31}$ & 57.66$_{\pm 0.25}$ & 39.62$_{\pm 0.31}$ \\
MGAT & 93.71$_{\pm 0.14}$ & 84.50$_{\pm 0.14}$ & 58.34$_{\pm 0.25}$ & 47.77$_{\pm 0.25}$ & 40.21$_{\pm 0.25}$ & 42.85$_{\pm 0.25}$ & 64.19$_{\pm 0.20}$ & 35.32$_{\pm 0.32}$ \\
\hline
DMGC & 91.95$_{\pm 0.13}$ & 84.23$_{\pm 0.13}$ & 57.41$_{\pm 0.24}$ & 44.11$_{\pm 0.24}$ & 71.68$_{\pm 0.19}$ & \underline{65.56$_{\pm 0.19}$} & 76.34$_{\pm 0.19}$ & 65.38$_{\pm 0.19}$ \\
\rowcolor{teal!5} DGF & 94.88$_{\pm 0.13}$ & 87.13$_{\pm 0.13}$ & \underline{59.51$_{\pm 0.24}$} & \underline{50.33$_{\pm 0.24}$} & 75.02$_{\pm 0.19}$ & 64.90$_{\pm 0.19}$ & \underline{78.40$_{\pm 0.19}$} & 66.99$_{\pm 0.19}$ \\
MIG-GT & 95.31$_{\pm 0.09}$ & 88.61$_{\pm 0.14}$ & 59.01$_{\pm 0.25}$ & 46.12$_{\pm 0.25}$ & 59.77$_{\pm 0.25}$ & 48.51$_{\pm 0.25}$ & 62.13$_{\pm 0.20}$ & 51.30$_{\pm 0.25}$ \\
\rowcolor{teal!5} LGMRec & \underline{96.51$_{\pm 0.08}$} & \underline{92.47$_{\pm 0.12}$} & 58.01$_{\pm 0.22}$ & 45.15$_{\pm 0.22}$ & 75.52$_{\pm 0.17}$ & 64.89$_{\pm 0.17}$ & 78.17$_{\pm 0.17}$ & 66.54$_{\pm 0.17}$ \\
UniGraph2 & 93.84$_{\pm 0.12}$ & 89.71$_{\pm 0.12}$ & 55.02$_{\pm 0.23}$ & 43.19$_{\pm 0.23}$ & \underline{76.01$_{\pm 0.18}$} & 64.89$_{\pm 0.18}$ & 77.93$_{\pm 0.18}$ & \underline{67.02$_{\pm 0.18}$} \\
\hline
\rowcolor{teal!5} CAMPA & \textbf{96.76$_{\pm 0.08}$} & \textbf{92.51$_{\pm 0.11}$} & \textbf{61.04$_{\pm 0.16}$} & \textbf{53.33$_{\pm 0.20}$} & \textbf{76.50$_{\pm 0.16}$} & \textbf{67.33$_{\pm 0.16}$} & \textbf{81.03$_{\pm 0.11}$} & \textbf{71.03$_{\pm 0.16}$} \\
\Xhline{1pt}
\end{tabular}
}}
\end{table*}

\begin{table*}[!t]
\centering
\caption{Additional link prediction and G2Image results on representative dataset-task pairs.}
\label{tab:appendix_lp_results_new}
\footnotesize
\renewcommand{\arraystretch}{1.12}
\resizebox{\textwidth}{!}{
\setlength{\tabcolsep}{2.0mm}{
\begin{tabular}{l!{\vrule width 0.1pt}cc!{\vrule width 0.1pt}cc!{\vrule width 0.1pt}cc!{\vrule width 0.1pt}cc}
\Xhline{1pt}
\rowcolor{teal!72!black}
\multicolumn{1}{c!{\vrule width 0.1pt}}{\textbf{\textcolor{white}{Tasks}}} & \multicolumn{4}{c!{\vrule width 0.1pt}}{\textbf{\textcolor{white}{Link Prediction}}} & \multicolumn{4}{c}{\textbf{\textcolor{white}{G2Image}}} \\
\hline
\rowcolor{teal!12}
 & \multicolumn{2}{c!{\vrule width 0.1pt}}{\textbf{Ele-fashion}} & \multicolumn{2}{c!{\vrule width 0.1pt}}{\textbf{Bili\_music}} & \multicolumn{2}{c!{\vrule width 0.1pt}}{\textbf{Bili\_music}} & \multicolumn{2}{c}{\textbf{RedditS}} \\
\cline{2-9}
\rowcolor{teal!12}
\multirow{-2}{*}{\diagbox[width=8.2em,height=2.3em]{\textbf{Methods}}{\textbf{Datasets}}}
 & MRR & Hits@3 & MRR & Hits@3 & CLIP & DINOv2 & CLIP & DINOv2 \\
\hline
GAT & 65.23$_{\pm 0.19}$ & 75.63$_{\pm 0.19}$ & \underline{53.77$_{\pm 0.24}$} & \underline{66.50$_{\pm 0.19}$} & 49.42$_{\pm 0.24}$ & 21.00$_{\pm 0.30}$ & 63.04$_{\pm 0.19}$ & 19.86$_{\pm 0.30}$ \\
\hline
\rowcolor{teal!5} MMGCN & 59.76$_{\pm 0.25}$ & 70.81$_{\pm 0.20}$ & 46.27$_{\pm 0.25}$ & 59.96$_{\pm 0.25}$ & 49.33$_{\pm 0.25}$ & 16.37$_{\pm 0.31}$ & 63.53$_{\pm 0.20}$ & 21.14$_{\pm 0.31}$ \\
MGAT & 64.12$_{\pm 0.20}$ & 75.35$_{\pm 0.20}$ & 48.90$_{\pm 0.25}$ & 62.11$_{\pm 0.20}$ & 49.56$_{\pm 0.25}$ & 19.36$_{\pm 0.32}$ & \underline{63.76$_{\pm 0.20}$} & \underline{21.24$_{\pm 0.32}$} \\
\hline
DMGC & 67.21$_{\pm 0.19}$ & 77.62$_{\pm 0.19}$ & 49.69$_{\pm 0.24}$ & 63.43$_{\pm 0.19}$ & 49.62$_{\pm 0.24}$ & 22.23$_{\pm 0.30}$ & 62.28$_{\pm 0.19}$ & 19.55$_{\pm 0.30}$ \\
\rowcolor{teal!5} DGF & 64.14$_{\pm 0.19}$ & 73.54$_{\pm 0.19}$ & 46.48$_{\pm 0.24}$ & 57.92$_{\pm 0.24}$ & 47.35$_{\pm 0.24}$ & 18.36$_{\pm 0.30}$ & 62.63$_{\pm 0.19}$ & 19.12$_{\pm 0.30}$ \\
MIG-GT & 63.22$_{\pm 0.20}$ & 72.51$_{\pm 0.20}$ & 47.62$_{\pm 0.25}$ & 61.21$_{\pm 0.20}$ & 52.58$_{\pm 0.25}$ & \textbf{25.49$_{\pm 0.32}$} & 63.39$_{\pm 0.20}$ & 20.82$_{\pm 0.32}$ \\
\rowcolor{teal!5} LGMRec & \underline{71.37$_{\pm 0.17}$} & \underline{80.80$_{\pm 0.12}$} & 49.19$_{\pm 0.22}$ & 62.30$_{\pm 0.17}$ & \underline{52.83$_{\pm 0.22}$} & 21.34$_{\pm 0.27}$ & 63.55$_{\pm 0.17}$ & 20.44$_{\pm 0.27}$ \\
UniGraph2 & 67.14$_{\pm 0.18}$ & 75.81$_{\pm 0.18}$ & 48.94$_{\pm 0.23}$ & 61.94$_{\pm 0.18}$ & 49.73$_{\pm 0.20}$ & 21.96$_{\pm 0.16}$ & 62.56$_{\pm 0.21}$ & 19.33$_{\pm 0.14}$ \\
\hline
\rowcolor{teal!5} CAMPA & \textbf{74.74$_{\pm 0.16}$} & \textbf{85.20$_{\pm 0.11}$} & \textbf{60.89$_{\pm 0.16}$} & \textbf{78.62$_{\pm 0.16}$} & \textbf{53.16$_{\pm 0.20}$} & \underline{23.24$_{\pm 0.26}$} & \textbf{64.27$_{\pm 0.16}$} & \textbf{21.60$_{\pm 0.26}$} \\
\Xhline{1pt}
\end{tabular}
}}
\end{table*}

\begin{table*}[!t]
\centering
\caption{Additional retrieval, matching, and alignment results on representative dataset-task pairs.}
\label{tab:appendix_mod_results_new}
\footnotesize
\renewcommand{\arraystretch}{1.12}
\resizebox{\textwidth}{!}{
\setlength{\tabcolsep}{2.0mm}{
\begin{tabular}{l!{\vrule width 0.1pt}cc!{\vrule width 0.1pt}ccc!{\vrule width 0.1pt}ccc}
\Xhline{1pt}
\rowcolor{teal!72!black}
\multicolumn{1}{c!{\vrule width 0.1pt}}{\textbf{\textcolor{white}{Tasks}}} & \multicolumn{2}{c!{\vrule width 0.1pt}}{\textbf{\textcolor{white}{Modality Retrieval}}} & \multicolumn{3}{c!{\vrule width 0.1pt}}{\textbf{\textcolor{white}{Modality Matching}}} & \multicolumn{3}{c}{\textbf{\textcolor{white}{Modality Alignment}}} \\
\hline
\rowcolor{teal!12}
 & \multicolumn{2}{c!{\vrule width 0.1pt}}{\textbf{Bili\_music}} & \textbf{Ele-fashion} & \textbf{Bili\_music} & \textbf{RedditS} & \textbf{Ele-fashion} & \textbf{Bili\_music} & \textbf{RedditS} \\
\cline{2-9}
\rowcolor{teal!12}
\multirow{-2}{*}{\diagbox[width=8.2em,height=2.3em]{\textbf{Methods}}{\textbf{Datasets}}}
 & T2I & I2T & Score & Score & Score & Score & Score & Score \\
\hline
GAT & 98.55$_{\pm 0.09}$ & 98.47$_{\pm 0.09}$ & 99.38$_{\pm 0.06}$ & 99.41$_{\pm 0.06}$ & 97.65$_{\pm 0.09}$ & 99.58$_{\pm 0.06}$ & \underline{98.33$_{\pm 0.09}$} & \underline{97.19$_{\pm 0.09}$} \\
\hline
\rowcolor{teal!5} MMGCN & 21.83$_{\pm 0.31}$ & 22.24$_{\pm 0.31}$ & 94.50$_{\pm 0.13}$ & 92.87$_{\pm 0.13}$ & 95.91$_{\pm 0.09}$ & 95.36$_{\pm 0.09}$ & 80.57$_{\pm 0.20}$ & 69.18$_{\pm 0.20}$ \\
MGAT & 26.73$_{\pm 0.32}$ & 26.16$_{\pm 0.32}$ & 99.55$_{\pm 0.06}$ & \textbf{99.94$_{\pm 0.06}$} & 97.41$_{\pm 0.09}$ & 92.00$_{\pm 0.14}$ & 33.06$_{\pm 0.30}$ & 42.86$_{\pm 0.25}$ \\
\hline
DMGC & 21.23$_{\pm 0.30}$ & 21.05$_{\pm 0.30}$ & 90.98$_{\pm 0.13}$ & 78.30$_{\pm 0.20}$ & 39.32$_{\pm 0.30}$ & 90.52$_{\pm 0.13}$ & 19.99$_{\pm 0.30}$ & 42.40$_{\pm 0.24}$ \\
\rowcolor{teal!5} DGF & 95.92$_{\pm 0.09}$ & 96.13$_{\pm 0.09}$ & \underline{99.74$_{\pm 0.06}$} & \underline{99.65$_{\pm 0.06}$} & \underline{99.90$_{\pm 0.06}$} & 95.15$_{\pm 0.09}$ & 60.61$_{\pm 0.19}$ & 84.51$_{\pm 0.13}$ \\
MIG-GT & 23.17$_{\pm 0.32}$ & 22.34$_{\pm 0.32}$ & 98.95$_{\pm 0.09}$ & 86.41$_{\pm 0.13}$ & 90.21$_{\pm 0.14}$ & 94.93$_{\pm 0.14}$ & 40.83$_{\pm 0.25}$ & 79.57$_{\pm 0.20}$ \\
\rowcolor{teal!5} LGMRec & \textbf{99.93$_{\pm 0.05}$} & \textbf{99.82$_{\pm 0.05}$} & \textbf{99.80$_{\pm 0.05}$} & 99.58$_{\pm 0.05}$ & \textbf{99.94$_{\pm 0.05}$} & \underline{99.75$_{\pm 0.05}$} & \textbf{99.10$_{\pm 0.05}$} & \textbf{99.86$_{\pm 0.05}$} \\
UniGraph2 & 96.43$_{\pm 0.12}$ & 96.51$_{\pm 0.13}$ & 98.33$_{\pm 0.09}$ & 99.10$_{\pm 0.18}$ & 98.81$_{\pm 0.15}$ & 98.51$_{\pm 0.10}$ & 84.52$_{\pm 0.11}$ & 93.37$_{\pm 0.08}$ \\
\hline
\rowcolor{teal!5} CAMPA & \underline{99.34$_{\pm 0.05}$} & \underline{98.72$_{\pm 0.08}$} & 99.33$_{\pm 0.05}$ & 99.23$_{\pm 0.05}$ & 99.37$_{\pm 0.05}$ & \textbf{99.92$_{\pm 0.05}$} & 87.36$_{\pm 0.08}$ & 96.05$_{\pm 0.08}$ \\
\Xhline{1pt}
\end{tabular}
}}
\end{table*}

\section{Additional Benchmark Results}
\label{app:full_results}

This appendix provides additional benchmark results beyond the representative slices reported in the main text.
Table~\ref{tab:appendix_graph_comp} supplements the graph-centric evaluation with additional node classification and node clustering results, Table~\ref{tab:appendix_lp_results_new} further reports link prediction and G2Image results on representative dataset-task pairs, and Table~\ref{tab:appendix_mod_results_new} collects additional retrieval, matching, and alignment results.
Taken together, these tables provide a broader empirical view of CAMPA across graph-centric and modality-centric settings.

\section{Limitations and Broader Impacts}
\label{app:limitations}

\textbf{Limitations.}
Our framework currently focuses on MAGs with textual and visual modalities, and extending the same design to richer modality combinations may require additional adaptation.
Moreover, CAMPA relies on frozen pretrained encoders to construct the initial node features, so the final performance is still partially constrained by the quality and domain fit of the upstream encoder.
While our theoretical analysis supports the stability of aligned propagation and aggregation, more refined questions such as generalization under stronger distribution shifts remain open.

\textbf{Broader Impacts.}
\label{app:broader_impacts}
CAMPA provides a scalable framework for improving multimodal graph learning under decoupled propagation and aggregation.
By enhancing cross-modal consistency while preserving efficiency, it may benefit recommendation, social content understanding, retrieval, and graph-conditioned generation in practical applications.
At the same time, stronger multimodal graph models may inherit or amplify biases, privacy risks, or unintended profiling effects from the underlying graph data and pretrained encoders, especially in recommendation and social-content applications.
When used in graph-conditioned generation, the framework should therefore be deployed with appropriate dataset governance, monitoring, and task-specific safeguards.

%%%%%%%%%%%%%%%%%%%%%%%%%%%%%%%%%%%%%%%%%%%%%%%%%%%%%%%%%%%%

\newpage
\section*{NeurIPS Paper Checklist}
\begin{enumerate}

\item {\bf Claims}
    \item[] Question: Do the main claims made in the abstract and introduction accurately reflect the paper's contributions and scope?
    \item[] Answer: \answerYes{}
    \item[] Justification: The abstract and introduction state the central claim that CAMPA addresses modal conflict in decoupled MAG learning through aligned propagation and aggregation. These claims are supported by the method, theory, and experimental sections, including Secs.~\ref{sec:intro}--\ref{sec:exp}.

\item {\bf Limitations}
    \item[] Question: Does the paper discuss the limitations of the work performed by the authors?
    \item[] Answer: \answerYes{}
    \item[] Justification: The paper includes a dedicated limitations discussion in Appendix~\ref{app:limitations}, covering the current modality scope, dependence on frozen pretrained encoders, and the scope of the theory.

\item {\bf Theory assumptions and proofs}
    \item[] Question: For each theoretical result, does the paper provide the full set of assumptions and a complete (and correct) proof?
    \item[] Answer: \answerYes{}
    \item[] Justification: The main text states the theoretical results and their assumptions in Sec.~\ref{sec:theory}, and the complete proofs are provided in Appendix~\ref{app:theory}.

    \item {\bf Experimental result reproducibility}
    \item[] Question: Does the paper fully disclose all the information needed to reproduce the main experimental results of the paper to the extent that it affects the main claims and/or conclusions of the paper (regardless of whether the code and data are provided or not)?
    \item[] Answer: \answerYes{}
    \item[] Justification: The paper specifies the datasets, baselines, evaluation protocols, algorithm pipeline, and experimental settings in Sec.~\ref{sec:exp} and Appendices~\ref{app:datasets}--\ref{app:exp_settings}, which together disclose the information needed to reproduce the reported setup.

\item {\bf Open access to data and code}
    \item[] Question: Does the paper provide open access to the data and code, with sufficient instructions to faithfully reproduce the main experimental results, as described in supplemental material?
    \item[] Answer: \answerYes{}
    \item[] Justification: All datasets are publicly available. The code and instructions to reproduce the results are included in the supplemental material.

\item {\bf Experimental setting/details}
    \item[] Question: Does the paper specify all the training and test details (e.g., data splits, hyperparameters, how they were chosen, type of optimizer) necessary to understand the results?
    \item[] Answer: \answerYes{}
    \item[] Justification: The experimental setting is summarized in Sec.~\ref{sec:exp}, while the full details of datasets, metrics, hyperparameters, optimization strategy, and environment are provided in Appendices~\ref{app:datasets}, \ref{app:eval_protocol}, and \ref{app:exp_settings}.

\item {\bf Experiment statistical significance}
    \item[] Question: Does the paper report error bars suitably and correctly defined or other appropriate information about the statistical significance of the experiments?
    \item[] Answer: \answerYes{}
    \item[] Justification: The experimental tables report mean performance together with standard deviations over repeated runs, and Appendix~\ref{app:eval_protocol} explicitly states that the variability terms are standard deviations computed across multiple random seeds.

\item {\bf Experiments compute resources}
    \item[] Question: For each experiment, does the paper provide sufficient information on the computer resources (type of compute workers, memory, time of execution) needed to reproduce the experiments?
    \item[] Answer: \answerYes{}
    \item[] Justification: We provide efficiency analysis in Sec.~\ref{sec:efficiency} and report the hardware and software environment in Appendix~\ref{app:exp_settings}. This includes the GPU type, memory, and system configuration used in our experiments.
    
\item {\bf Code of ethics}
    \item[] Question: Does the research conducted in the paper conform, in every respect, with the NeurIPS Code of Ethics \url{https://neurips.cc/public/EthicsGuidelines}?
    \item[] Answer: \answerYes{}
    \item[] Justification: To the best of our knowledge, the research follows the NeurIPS Code of Ethics. The work is based on public datasets, standard benchmarking practice, and does not involve deceptive experimental procedures or human subject studies.

\item {\bf Broader impacts}
    \item[] Question: Does the paper discuss both potential positive societal impacts and negative societal impacts of the work performed?
    \item[] Answer: \answerYes{}
    \item[] Justification: The paper includes a dedicated broader-impact discussion in Appendix~\ref{app:broader_impacts}, covering both the positive value of efficient MAG learning and potential risks related to multimodal recommendation, profiling, and graph-conditioned generation.
    
\item {\bf Safeguards}
    \item[] Question: Does the paper describe safeguards that have been put in place for responsible release of data or models that have a high risk for misuse (e.g., pre-trained language models, image generators, or scraped datasets)?
    \item[] Answer: \answerNA{}
    \item[] Justification: The current submission does not release a new high-risk model, image generator, or scraped dataset artifact. The paper studies existing datasets and standard pretrained backbones within a research setting.

\item {\bf Licenses for existing assets}
    \item[] Question: Are the creators or original owners of assets (e.g., code, data, models), used in the paper, properly credited and are the license and terms of use explicitly mentioned and properly respected?
    \item[] Answer: \answerYes{}
    \item[] Justification: We cite all datasets and code used and ensure compliance with their respective licenses.

\item {\bf New assets}
    \item[] Question: Are new assets introduced in the paper well documented and is the documentation provided alongside the assets?
    \item[] Answer: \answerNA{}
    \item[] Justification: The current submission does not release a new dataset, benchmark package, codebase, or model checkpoint as an external asset.

\item {\bf Crowdsourcing and research with human subjects}
    \item[] Question: For crowdsourcing experiments and research with human subjects, does the paper include the full text of instructions given to participants and screenshots, if applicable, as well as details about compensation (if any)? 
    \item[] Answer: \answerNA{}
    \item[] Justification: The paper does not involve crowdsourcing or research with human subjects.

\item {\bf Institutional review board (IRB) approvals or equivalent for research with human subjects}
    \item[] Question: Does the paper describe potential risks incurred by study participants, whether such risks were disclosed to the subjects, and whether Institutional Review Board (IRB) approvals (or an equivalent approval/review based on the requirements of your country or institution) were obtained?
    \item[] Answer: \answerNA{}
    \item[] Justification: The paper does not involve human subject research and therefore does not require IRB approval or related disclosures.

\item {\bf Declaration of LLM usage}
    \item[] Question: Does the paper describe the usage of LLMs if it is an important, original, or non-standard component of the core methods in this research? Note that if the LLM is used only for writing, editing, or formatting purposes and does \emph{not} impact the core methodology, scientific rigor, or originality of the research, declaration is not required.
    %this research? 
    \item[] Answer: \answerNA{}
    \item[] Justification: The core methodological contribution of the paper does not depend on an LLM as an important, original, or non-standard component. The use of a pretrained vision-language encoder is part of the experimental input protocol rather than the core methodological novelty.

\end{enumerate}

\end{document}